\documentclass[twoside,11pt]{article}

\usepackage{jmlr2e}

\usepackage{amsmath,amssymb} 
\usepackage{bm} 
\usepackage{bbm} 

\newtheorem{assumption}{Assumption}

\usepackage{scalerel}

\usepackage{tikz}
\makeatletter 
\newcommand*\bigcdot{\mathpalette\bigcdot@{.5}}
\newcommand*\bigcdot@[2]{\mathbin{\vcenter{\hbox{\scalebox{#2}{$\m@th#1\bullet$}}}}}
\makeatother

\newcommand{\vertiii}[1]{{\left\vert\kern-0.25ex\left\vert\kern-0.25ex\left\vert #1 
    \right\vert\kern-0.25ex\right\vert\kern-0.25ex\right\vert}}
    
\newcommand{\scalp}[1]{ \left\langle #1\right\rangle} %
\newcommand{\bNorm}[1]{\left\Vert #1\right\Vert} 
\newcommand{\bPc}[1]{\ensuremath{\left\{#1 \right\}}} 
\newcommand{\bPe}[1]{\ensuremath{\left[#1 \right]}} 
\newcommand{\bPr}[1]{\ensuremath{\left(#1 \right)}} 

  \newcommand{\Thstar}{{\Theta^\star}}  

\newcommand{\IE}{{\mathbb{E}}}\newcommand{\IF}{{\mathbb{F}}}  

\newcommand{\IP}{{\mathbb{P}}}  \newcommand{\IR}{{\mathbb{R}}}

 \def\D{\mathcal{D}}

\def\J{\mathcal{J}}\newcommand{\Lc}{\mathcal{L}}
\def\M{\mathcal{M}}\def\P{\mathcal{P}}
\def\S{\mathcal{S}}\def\V{\mathcal{V}}\def\X{\mathcal{X}}\def\Y{\mathcal{Y}}

\newcommand{\ga}{{\gamma}} \newcommand{\ep}{{\varepsilon}}
\newcommand{\la}{{\lambda}}  \newcommand{\kpa}{{\kappa}}    

\let\phi\varphi \newcommand\vphi\varphi

\newcommand{\tand}{\text{ and } }

\newcommand{\subjto}{\text{s.t.} }

\let\emptyset\varnothing

\newcommand{\oneVec}{\mathbbm{1}}

\def\gdw{\Longleftrightarrow}

\def\to{\rightarrow}

\DeclareMathOperator{\supp}{supp}

\DeclareMathOperator{\Id}{Id}

\DeclareMathOperator{\sign}{sign}

\DeclareMathOperator{\rank}{rank}
\DeclareMathOperator{\coh}{coh}
\DeclareMathOperator{\trace}{tr}

\DeclareMathOperator{\tr}{tr}

\DeclareMathOperator{\argmin}{argmin}

\def\wrt{w.r.t.~}

\newcommand{\ganorm}[1]{\left\Vert #1\right\Vert_\gamma} 

\jmlrheading{1}{2019}{1-48}{9/18}{10/18}{nussbaum}{Frank Nussbaum and
  Joachim Giesen}

\ShortHeadings{Ising Models with Latent Conditional Gaussian
  Variables}{Nussbaum and Giesen} \firstpageno{1}

\begin{document}

\title{Ising Models with Latent Conditional Gaussian Variables}

\author{\name Frank Nussbaum \email frank.nussbaum@uni-jena.de \\
  \addr Institut f\"ur Informatik\\
  Friedrich-Schiller-Universit\"at Jena\\
  Germany
  \AND
  \name Joachim Giesen \email joachim.giesen@uni-jena.de \\
  \addr Institut f\"ur Informatik\\
  Friedrich-Schiller-Universit\"at Jena\\
  Germany}

\editor{Satyen Kale and Aur\'{e}lien Garivier}

\maketitle

\begin{abstract}
  Ising models describe the joint probability distribution of a vector
  of binary feature variables. Typically, not all the variables
  interact with each other and one is interested in learning the
  presumably sparse network structure of the interacting
  variables. However, in the presence of latent variables, the
  conventional method of learning a sparse model might fail. This is
  because the latent variables induce indirect interactions of the
  observed variables.  In the case of only a few latent conditional
  Gaussian variables these spurious interactions contribute an
  additional low-rank component to the interaction parameters of the
  observed Ising model.  Therefore, we propose to learn a sparse +
  low-rank decomposition of the parameters of an Ising model using a
  convex regularized likelihood problem.  We show that the same
  problem can be obtained as the dual of a maximum-entropy problem
  with a new type of relaxation, where the sample means collectively
  need to match the expected values only up to a given tolerance. The
  solution to the convex optimization problem has consistency
  properties in the high-dimensional setting, where the number of
  observed binary variables and the number of latent conditional
  Gaussian variables are allowed to grow with the number of training
  samples.
\end{abstract}


\section{Introduction}

The principle of maximum entropy was proposed by
\cite{jaynes1957information} for probability density estimation.  It
states that from the probability densities that represent the current
state of knowledge one should choose the one with the largest entropy,
that is, the one which does not introduce additional biases. The state
of knowledge is often given by sample points from a sample space and
some fixed functions (sufficient statistics) on the sample space. The
knowledge is then encoded naturally in form of constraints on the
probability density by requiring that the expected values of the
functions equal their respective sample means.  Here, we assume the
particularly simple multivariate sample space $\mathcal{X} =
\{0,1\}^d$ and functions
\[
 \varphi_{ij}: x\mapsto x_ix_j\quad \textrm{ for }\: i,j\in
        [d]=\{1,\ldots, d\}.
\]
Suppose we are given sample points
$x^{(1)},\ldots,x^{(n)}\in\mathcal{X}$.  Then formally, for estimating
the distribution from which the sample points are drawn, the principle
of maximum entropy suggests solving the following entropy maximization
problem
\[
  \max_{p\,\in\, \mathcal{P}} \: H(p) \quad \textrm{ s.t. }\:
  \mathbb{E}[\varphi_{ij}] = \frac{1}{n} \sum_{k=1}^n
  \varphi_{ij}(x^{(k)})\: \textrm{ for all } i,j\in [d],
\]
where $\mathcal{P}$ is the set of all probability distributions on
$\mathcal{X}$, the expectation is with respect to the distribution
$p\in\mathcal{P}$, and $H(p)=- \IE[\log p(x)]$ is the entropy.  We
denote the $(d\times d)$-matrix $\big(\frac{1}{n} \sum_{k=1}^n
\varphi_{ij} (x^{(k)})\big)_{i,j\in [d]}$ of sample means compactly by
$\Phi^n$ and the matrix of functions $\big(\varphi_{ij}\big)_{i,j\in
  [d]}$ by $\Phi$. Then, the entropy maximization problem becomes
\[
  \max_{p\,\in\,\mathcal{P}}\: H(p) \quad \textrm{ s.t. }\:
  \mathbb{E}[\Phi] - \Phi^n = 0.
\]
\cite{dudik2004performance} observed that invoking the principle of
maximum entropy tends to overfit when the number of features $d$ is
large. Requiring that the expected values of the functions equal their
respective sample means can be too restrictive. Consequently, they
proposed to relax the constraint using the maximum norm as
\[
 \|\mathbb{E}[\Phi] - \Phi^n\|_\infty \leq c
\]
for some $c>0$. That is, for every function the expected value only
needs to match the sample mean up to a tolerance of $c$.  The dual of
the relaxed problem has a natural interpretation as a
feature-selective $\ell_1$-regularized log-likelihood maximization
problem
\[
\max_{S\,\in\,\textrm{Sym}(d)} \quad \ell (S) - c\|S\|_1,
\]
where $\textrm{Sym}(d)$ is the set of symmetric $(d\times
d)$-matrices, $S\in\textrm{Sym}(d)$ is the matrix of dual
variables for the constraint $\|\mathbb{E}[\Phi] - \Phi^n\|_\infty
\leq c$, and 
\[
\ell (S) = \scalp{S,  \Phi^n} -a(S)
\]
is the log-likelihood function for pairwise Ising models with the
standard matrix dot product $\scalp{S, \Phi^n} =
\tr\big(S^\top{\Phi^n}\big)$ and normalizer (log-partition function)
\[
a(S) = \log \sum_{x\in \mathcal{X}} \scalp{S, \Phi(x)}.
\]

In this paper, we are restricting the relaxation of the entropy
maximization problem by also enforcing the alternative constraint
\[
\|\mathbb{E}[\Phi] - \Phi^n\| \leq \lambda,
\]
where $\lambda >0$ and $\|\cdot\|$ denotes the spectral norm
on $\textrm{Sym}(d)$. A difference to the maximum norm constraint
is that now the expected values of the functions only need to
collectively match the sample means up to a tolerance of $\lambda$
instead of individually. The dual of the more strictly relaxed entropy
maximization problem
\[
  \max_{p\,\in\,\mathcal{P}}\:  H(p) \quad
  \textrm{ s.t. }\:  \|\mathbb{E}[\Phi] - \Phi^n\|_\infty \leq c 
  \,\textrm{ and }\, \|\mathbb{E}[\Phi] - \Phi^n\| \leq \lambda
\]
is the regularized log-likelihood maximization problem
\[
  \max_{S,\,L_1,\,L_2\,\in \,\textrm{Sym}(d)} \: \ell (S+L_1-L_2) -
  c\|S\|_1 - \lambda \tr (L_1+L_2) \quad \textrm{ s.t. }\: L_1,L_2
  \succeq 0,
\]
see Appendix~\ref{s_appendix_dual}.  Here, the regularization term
$\tr (L_1+L_2)$ promotes a low rank of the positive-semidefinite
matrix $L_1+L_2$. This implies that the matrix $L_1-L_2$ in the
log-likelihood function also has low rank. Thus, a solution of the dual
problem is the sum of a sparse matrix $S$ and a low-rank matrix
$L_1-L_2$. This can be interpreted as follows: the variables interact
indirectly through the low-rank matrix $L_1-L_2$, while some of the
direct interactions through the matrix $S$ are turned off by setting
entries in $S$ to zero.  We get a more intuitive interpretation of the
dual problem if we consider a weakening of the spectral norm
constraint. The spectral norm constraint is equivalent to the two
constraints
\[
\mathbb{E}[\Phi]-\Phi^n \preceq \lambda \Id \quad\textrm{ and }\quad
\Phi^n-\mathbb{E}[\Phi]  \preceq \lambda \Id
\]
that bound the spectrum of the matrix $\mathbb{E}[\Phi] - \Phi^n$ from
above and below. If we replace the spectral norm constraint by only
the second of these two constraints in the maximum-entropy problem,
then the dual problem becomes
\[
  \max_{S,\,L\,\in\, \textrm{Sym}(d)}\:\ell (S+L) - c\|S\|_1 -
  \lambda \tr (L) \quad \textrm{ s.t. }\: L \succeq 0.
\]
This problem also arises as the log-likelihood maximization problem
for a conditional Gaussian model (see~\cite{lauritzen1996graphical})
that exhibits observed binary variables and unobserved, latent
conditional Gaussian variables. The sample space of the full mixed
model is $\mathcal{X}\times\mathcal{Y} = \{0,1\}^d \times
\mathbb{R}^{l}$, where $\mathcal{Y}=\mathbb{R}^{l}$ is the sample
space for the unobserved variables. We want to write down the density
of the conditional Gaussian model on this sample space. For that we
respectively denote the interaction parameters between the observed
binary variables by $S\in \textrm{Sym}(d)$, the ones between the
observed binary and latent conditional Gaussian variables by
$R\in\mathbb{R}^{l\times d}$, and the ones between the latent
conditional Gaussian variables by $\Lambda\in \textrm{Sym}(l)$, where
$\Lambda\succ0$. Then, for $(x,y)\in \mathcal{X}\times\mathcal{Y}$
and up to normalization, the density of the conditional Gaussian
model is given as
\[
p(x,y)  \propto \exp \left( x^\top S x + y^\top Rx - \frac{1}{2} y^\top
\Lambda y \right).
\]
One can check, see also~\cite{lauritzen1996graphical}, that the
conditional densities $p(y\:|\:x)$ are $l$-variate Gaussians on
$\Y$. Here, we are interested in the marginal distribution
\[
p(x) \propto \exp \left( \Big\langle{S+\frac{1}{2}R^\top \Lambda^{-1} R,
  \Phi(x)}\Big\rangle \right)
\]
on $\mathcal{X}$ that is obtained by integrating over the unobserved
variables in $\mathcal{Y}$, see Appendix \ref{s_appendix_marg}.  The
matrix $L=\frac{1}{2}R^\top \Lambda^{-1} R$ is symmetric and positive
semidefinite. The log-likelihood function for the marginal model and
the given data is thus given as
\[
\ell(S+L) = \scalp{S+L, \Phi^n} -a(S+L),
\]
where $S,L\in \textrm{Sym}(d)$, $L\succeq 0$ and $a(S+L)$ is once
again the normalizer of the density.
\medskip

If only a few of the binary variables interact directly, then $S$ is
sparse, and if the number of unobserved variables $l$ is small
compared to $d$, then $L$ is of low rank. Hence, one could attempt to
recover $S$ and $L$ from the data using the regularized log-likelihood
maximization problem
\begin{equation}
\max_{S,\,L\,\in\,\textrm{Sym}(d)}\: \ell (S+L) - c\|S\|_1 - \lambda \tr
(L) \quad \textrm{ s.t. }\: L \succeq 0 \tag{ML} \label{MLProblem}
\end{equation}
that we encountered before.
\medskip

We are now in a similar situation as has been discussed
by~\cite{ChandrasekaranPW12} who studied Gaussian graphical models
with latent Gaussian variables. They were able to consistently
estimate both the number of latent components, in our case $l$, and
the conditional graphical model structure among the observed
variables, in our case the zeroes in $S$. Their result holds in the
high-dimensional setting, where the number of variables (latent and
observed) may grow with the number of observed sample points.  Here,
we show a similar result for the Ising model with latent conditional
Gaussian variables, that is, the one that we have introduced above.

\section{Related Work}

\emph{Graphical Models.} The introduction of decomposed sparse +
low-rank models followed a period of quite extensive research on
sparse graphical models in various settings, for example Gaussians
(\cite{meinshausen2006high}, \cite{ravikumar2011high}), Ising models
(\cite{ravikumar2010high}), discrete models
(\cite{jalali2011onlearnin}), and more general conditional Gaussian
and exponential family models (\cite{lee2015learning}, \cite{LST13},
\cite{cheng2017high}). All estimators of sparse graphical models
maximize some likelihood including a $\ell_1$-penalty that induces
sparsity.
\medskip

Most of the referenced works contain high-dimensional consistency
analyses that particularly aim at the recovery of the true graph
structure, that is, the information which variables are \emph{not}
conditionally independent and thus interact. A prominent proof
technique used throughout is the primal-dual-witness method originally
introduced in \cite{wainwright2009sharp} for the \textsc{LASSO}, that is,
sparse regression. Generally, the assumptions necessary in order to be
able to successfully identify the true interactions for graphical
models (or rather the active predictors for the \textsc{LASSO}) are
very similar. For example, one of the conditions that occurs
repeatedly is irrepresentability, sometimes also referred to as
incoherence. Intuitively, this condition limits the influence
the active terms (edges) can have on the inactive terms (non-edges),
see \cite{ravikumar2011high}.
\medskip

\emph{Sparse + low-rank models.} The seminal work of
\cite{ChandrasekaranPW12} is the first to propose learning sparse +
low-rank decompositions as an extension of classical graphical
models. As such it has received a lot of attention since then, putting
forth various commentators, for example \cite{candes2012discussion},
\cite{laur2012discussion}, and
\cite{wainwright2012discussion}. Notably, \cite{ChandrasekaranPW12}'s
high-dimensional consistency analysis generalizes the proof-technique
previously employed in graphical models. Hence, unsurprisingly, one of
their central assumptions is a generalization of the
irrepresentability condition.
\medskip

Astoundingly, not so much effort has been undertaken in generalizing
sparse + low-rank models to broader domains of variables. The
particular case of multivariate binary models featuring a sparse +
low-rank decomposition is related to Item Response Theory (IRT, see
for example \cite{hambleton1991fundamentals}). In IRT the observed
binary variables (test items) are usually assumed to be conditionally
independent given some continuous latent variable (trait of the test
taker). \cite{chen2018robust} argued that measuring conditional
dependence by means of sparse + low-rank models might improve results
from classical IRT. They estimate their models using
pseudo-likelihood, a strategy that they also proposed in an earlier
work, see~\cite{chen2016fused}.
\medskip

\cite{chen2016fused} show that their estimator recovers the algebraic
structure, that is, the conditional graph structure and the number of
latent variables, with probability tending to one. However, their
analysis only allows a growing number of sample points whereas they
keep the number of variables fixed. Their result thus severs from the
tradition to analyze the more challenging high-dimensional setting,
where the number of variables is also explicitly tracked.
\medskip

\emph{Placement of our work.} Our main contribution is a
high-dimensional consistency analysis of a likelihood estimator for
multivariate binary sparse + low-rank models. Furthermore, our
analysis is the first to show parametric consistency of the
likelihood-estimates and to provide explicit rates for this type of
models. It thus complements the existing literature.  Our other
contribution is the connection to a particular type of relaxed
maximum-entropy problems that we established in the introduction. We
have shown that this type of relaxation leads to an interpretation as
the marginal model of a conditional Gaussian
distribution. Interestingly, this has not drawn attention before,
though our semidefiniteness constraints can be obtained as special
cases of the general relaxed maximum-entropy problem discussed
in~\cite{dudik2006maximum}.

\section{Parametric and Algebraic Consistency}

This section constitutes the main part of this paper. Here, we discuss
assumptions that lead to consistency properties of the solution to the
likelihood problem~\ref{MLProblem} and state our consistency result. We
are interested in the high-dimensional setting, where the number of
samples~$n$, the number of observed binary variables~$d$, and the
number of latent conditional Gaussian variables~$l$ are allowed to
grow simultaneously.  Meanwhile, there are some other problem-specific
quantities that concern the curvature of the problem that we assume to
be fixed. Hence, we keep the geometry of the problem fixed.
\medskip

For studying the consistency properties, we use a slight reformulation
of Problem~\ref{MLProblem} from the introduction. First, we switch from
a maximization to a minimization problem, and let $\ell$ be the
\emph{negative} log-likelihood from now on. Furthermore, we change the
representation of the regularization parameters, namely
\begin{equation}
\begin{array}{lrlr}
 (S_n, L_n) \; = &\underset{S,\,
    L} {\argmin}& \ell(S+L) + \la_n \bPr{\gamma \|S\|_{1} + \tr L}
  \\ &\subjto & L\succeq 0,&
\end{array} \tag{SL}  \label{prob_SL}
\end{equation}
where $\gamma$ controls the trade-off between the two regularization
terms and $\la_n$ controls the trade-off between the negative
log-likelihood term and the regularization terms.
\medskip

We want to point out that our consistency proof follows the lines of
the seminal work in~\cite{ChandrasekaranPW12} who investigate a convex
optimization problem for the parameter estimation of a model with
observed and latent Gaussian variables. The main difference to the
Ising model is that the Gaussian case requires a positive-definiteness
constraint on the pairwise interaction parameter matrix $S+L$ that is
necessary for normalizing the density. Furthermore, in the Gaussian
case the pairwise interaction parameter matrix $S+L$ is the inverse of
the covariance matrix. This is no longer the case
for the Ising model, see~\cite{loh2012structure}.
\medskip

In this work, we want to answer the question if it is possible to
recover the parameters from data that has been drawn from a
hypothetical \emph{true} model distribution parametrized by $S^\star$
and $L^\star$. We focus on two key concepts of successful recovery in
an asymptotic sense with high probability. The first is
\emph{parametric} consistency. This means that $(S_n, L_n)$ should be
close to $(S^\star, L^\star)$ \wrt some norm.  Since the regularizer
is the composed norm $\gamma \|S\|_{1} + \tr L$, a natural norm for
establishing parametric consistency is its dual norm
\begin{align*}
\ganorm{(S, L)} &= \max\bPc{\frac{\|S\|_{\infty}}{\gamma}, \|L\|}.
\end{align*}
The second type of consistency that we study is \emph{algebraic}
consistency. It holds if $S_n$ recovers the true sparse support
of $S^\star$, and if $L_n$ has the same rank as $L^\star$.
\medskip

In the following we discuss the assumptions for our consistency
result. For that we proceed as follows: First, we discuss the
requirements for parametric consistency of the compound matrix in
Section~\ref{sec:paramcons_compound}. Next, we work out the three
central assumptions that are sufficient for individual recovery of
$S^\star$ and $L^\star$ in Section~\ref{sec:assumptions}.  We state
our consistency result in Section~\ref{sec:thm}. Finally, in
Section~\ref{sec:outline} we outline the proof, the details of which
can be found in Section~\ref{app:proof}.

\subsection{Parametric consistency of the compound matrix}
\label{sec:paramcons_compound}

In this section, we briefly sketch how the negative log-likelihood
part of the objective function in Problem~\ref{prob_SL} drives the
compound matrix $\Theta_n = S_n + L_n$ that is constructed from the
solution $(S_n,L_n)$ to parametric consistency with high probability.
We only consider the negative log-likelihood part because we assume
that the relative weight $\la_n$ of the regularization terms in the
objective function goes to zero as the number of sample points goes to
infinity. This implies that the estimated compound matrix is not
affected much by the regularization terms since they contribute mostly
small (but important) adjustments. More specifically, the
$\ell_1$-norm regularization on $S$ shrinks entries of $S$ such that
entries of small magnitude are driven to zero such that $S_n$ will
likely be a sparse matrix. Likewise, the trace norm (or nuclear norm)
can be thought of diminishing the singular values of the matrix $L$
such that small singular values become zero, that is, $L_n$ will
likely be a low-rank matrix.
\medskip

The negative log-likelihood function is strictly convex and thus has a
unique minimizer $\hat \Theta$. We can assume that $\hat \Theta
\approx \Theta_n$. Let $\Thstar = S^\star + L^\star$ and
$\Delta_\Theta = \hat \Theta-\Thstar$. Then, consistent recovery of
the compound matrix $\Thstar$ is essentially equivalent to the
estimation error $\Delta_\Theta$ being small.  Now, consider the
Taylor expansion
\[
 \ell(\Thstar + \Delta_\Theta) = \ell(\Thstar) + \nabla
 \ell(\Thstar)^\top \Delta_\Theta + \frac{1}{2}\Delta_\Theta^\top
 \nabla^2 \ell(\Thstar) \Delta_\Theta + R(\Delta_\Theta)
\]
with remainder $R(\Delta_\Theta)$. It turns out that if the
number of samples is sufficiently large, then the gradient $\nabla
\ell(\Thstar)$ is small with high probability, and if $\Delta_\Theta$ is 
small, then the remainder $R(\Delta_\Theta)$ is also small. In this case, 
the Taylor expansion implies that locally around the true parameters
the negative log-likelihood is well approximated by the quadratic form
induced by its Hessian, namely
\[
\ell(\Thstar + \Delta_\Theta) \approx \ell(\Thstar) +
\frac{1}{2}\Delta_\Theta^\top \nabla^2 \ell(\Thstar) \Delta_\Theta.
\]
This quadratic form is obviously minimized at $\Delta_\Theta=0$, which
would entail consistent recovery of $\Thstar$ in a parametric
sense. However, this does not explain how the sparse and low-rank
components of $\Thstar$ can be recovered consistently. In the next
section we elaborate sufficient assumptions for the consistent
recovery of these components.

\subsection{Assumptions for individual recovery}
\label{sec:assumptions}

Consistent recovery of the components, more specifically parametric
consistency of the solutions $S_n$ and $L_n$, requires the two errors
$\Delta_S = S_n -S^\star$ and $\Delta_L = L_n - L^\star$ to be small
(in their respective norms). Both errors together form the joint error
$\Delta_S+\Delta_L = \Theta_n - \Theta^\star \approx \Delta_\Theta$. 
Note though that the minimum of
the quadratic form from the previous section at $\Delta_\Theta=0$ does
not imply that the individual errors $\Delta_S$ and $\Delta_L$ are
small.  We can only hope for parametric consistency of $S_n$
and $L_n$ if they are the unique solutions to Problem~\ref{prob_SL}.
\medskip

For uniqueness of the solutions we need to study optimality
conditions. Problem~\ref{prob_SL} is the Lagrange form of
the constrained problem
\[
\min \,\ell(S+L) \quad \textrm{ s.t. }\: \|S\|_{1} \leq c_n \textrm{ and }
\|L\|_\ast \leq t_n
\]
for suitable regularization parameters $c_n$ and $t_n$, where we have
neglected the positive-semidefiniteness constraint on $L$. The
constraints can be thought of as convex relaxations of constraints
that require $S$ to have a certain sparsity and require $L$ to have at
most a certain rank. That is, $S$ should be contained in the set of
symmetric matrices of a given sparsity and $L$ should be contained in
the set of symmetric low-rank matrices.  To formalize these sets we
briefly review the varieties of sparse and low-rank matrices.

\paragraph{Sparse matrix variety.}

For $M\in \textrm{Sym}(d)$ the support is defined as
\[
\supp(M) = \{(i,j)\in [d]\times[d]:\; M_{ij}\neq0\}, 
\]
and the variety of sparse symmetric matrices with at most $s$
non-zero entries is given as
\begin{align*}
  \S(s) =  \{S\in \textrm{Sym}(d) : |\supp(S)| \leq s  \}.
\end{align*}
Any matrix $S$ with $|\supp (S)|=s$ is a smooth point of $\S(s)$ with
tangent space
\begin{align*}
  \Omega (S) &= \{M\in \textrm{Sym}(d):\; \supp(M) \subseteq
  \supp(S)\}.
\end{align*}

\paragraph{Low-rank matrix variety.}

The variety of matrices with rank at most $r$ is given as
\begin{align*}
  \Lc(r) = \{L\in \textrm{Sym}(d):\; \rank(L)\leq r\}.
\end{align*}
Any matrix $L$ with rank $r$ is a smooth point of $\Lc(r)$ with
tangent space
\begin{align*}
  T(L) = \bPc{U X^\top + X U^\top:\; X\in\IR^{d\times r}},
\end{align*}
where $L=UD U^\top$ is the restricted eigenvalue decomposition
of $L$, that is, $U\in\IR^{d\times r}$ has orthonormal columns and
$D\in\IR^{r\times r}$ is diagonal.
\medskip

Next, we formulate conditions that ensure uniqueness in terms of the
tangent spaces of the introduced varieties.

\paragraph{Transversality.}

Remember that we understand the constraints in the constrained
formulation of Problem~\ref{prob_SL} as convex relaxations of
constraints of the form $S\in\S (s)$ and $L\in \Lc(r)$.  Because the
negative log-likelihood function $\ell$ is a function of $S+L$, its
gradient with respect to $S$ and its gradient with respect to $L$
coincide at $S+L$. Hence, the first-order optimality conditions for
the non-convex problem require that the gradient of the negative
log-likelihood function needs to be normal to $\S (s)$ and $\Lc(r)$ at
any (locally) optimal solutions $\hat S$ and $\hat L$, respectively.
If the solution $(\hat S,\hat L)$ is not (locally) unique, then
basically the only way to get an alternative optimal solution that
violates (local) uniqueness is by translating $\hat S$ and $\hat L$ by
an element that is tangential to $\S(s)$ at $\hat S$ and tangential to
$\Lc(r)$ at $\hat L$, respectively. Thus, it is necessary for (local)
uniqueness of the optimal solution that such a tangential direction
does not exist. Hence, the tangent spaces $\Omega (\hat S)$ and
$T(\hat L)$ need to be \emph{transverse}, that is, $\Omega (\hat S)
\cap T(\hat L) =\{0\}$. Intuitively, if we require that transversality
holds for the true parameters $(S^\star,L^\star)$, that is, $\Omega
(S^\star) \cap T(L^\star) =\{0\}$, then provided that $(\hat S,\hat
L)$ is close to $(S^\star,L^\star)$, the tangent spaces $\Omega (\hat
S)$ and $T(\hat L)$ should also be transverse.
\medskip

We do not require transversality explicitly since it is implied by
stronger assumptions that we motivate and state in the
following. In particular, we want the (locally) optimal solutions
$\hat S$ and $\hat L$ not only to be unique, but also to be
\emph{stable} under perturbations. This stability needs some
additional concepts and notation that we introduce now.

\paragraph{Stability assumption.}

Here, stability means that if we perturb $\hat S$ and $\hat L$ in the
respective tangential directions, then the gradient of the negative
log-likelihood function should be far from being normal to the sparse
and low-rank matrix varieties at the perturbed $\hat S$ and $\hat L$,
respectively. As for transversality, we require stability for the true
solution $(S^\star, L^\star)$ and expect that it carries over to the
optimal solutions $\hat S$ and $\hat L$, provided they are close.  More
formally, we consider perturbations of $S^\star$ in directions from the
tangent space $\Omega=\Omega(S^\star)$, and perturbations of $L^\star$
in directions from tangent spaces to the low-rank variety that are close
to the true one $T=T(L^\star)$. The reason for considering tangent
spaces close to $T(L^\star)$ is that there are low-rank matrices close
to $L^\star$ that are not contained in $T(L^\star)$ because the
low-rank matrix variety is locally curved at any smooth point.
\medskip

Now, in light of a Taylor expansion the change of the gradient is
locally governed by the data-independent Hessian $H^\star = \nabla^2
\ell(\Thstar) = \nabla^2 a(\Thstar)$ of the negative log-likelihood
function at $\Theta^\star$. To make sure that the gradient of the
tangentially perturbed (true) solution cannot be normal to the
respective matrix varieties we require that it has a significant
component in the tangent spaces at the perturbed solution. This is
achieved if the \emph{minimum gains} of the Hessian $H^\star$ in the
respective tangential directions
\begin{align*}
\alpha_\Omega &= \underset{M\in \Omega,\,\|M\|_{\infty}=1}{\min} \:
\|P_\Omega H^\star M\|_{\infty}, \quad\textrm{and} \\
\alpha_{T,\varepsilon} &= \underset{\rho(T, T')\leq
  \varepsilon}{\min} \:\: \underset{M\in T',\, \|M\|=1}{\min} \:
\|P_{T'} H^\star M\|
\end{align*}
are large, where $T'\subseteq \textrm{Sym}(d)$ are tangent spaces to
the low-rank matrix variety that are close to $T$ in terms of the
\emph{twisting}
\begin{align*}
\rho(T, T') = \underset{\|M\| =1}{\max} \bNorm {\bPe{P_{T}-P_{T'}}(M)]}
\end{align*}
between these subspaces given some $\varepsilon>0$. Here, we denote
projections onto a matrix subspace by $P$ subindexed by the subspace.
\medskip

Note though that only requiring $\alpha_\Omega$ and
$\alpha_{T,\varepsilon}$ to be large is not enough if the
\emph{maximum effects} of the Hessian $H^\star$ in the respective
normal directions
\begin{align*}
  \delta_\Omega &= \underset{M\in \Omega,\,\|M\|_{\infty}=1}{\max} \:
  \|P_{\Omega^\perp} H^\star M\|_{\infty},\quad\textrm{and}\\
  \delta_{T,\varepsilon} &= \underset{\rho(T, T')\leq\varepsilon}{\max}
  \:\: \underset{M\in T',\, \|M\|=1}{\max}
  \|P_{T'^\perp} H^\star M\|
\end{align*}
are also large, because then the gradient of the negative
log-likelihood function at the perturbed (true) solution could still
be almost normal to the respective varieties.  Here, $\Omega^\perp$ is
the normal space at $S^\star$ orthogonal to $\Omega$, and $T'^\perp$
is the space orthogonal to $T'$. 
\medskip

Overall, we require that $\alpha_\varepsilon = \min \{ \alpha_\Omega,
\alpha_{T,\varepsilon}\}$ is bounded away from zero and that the ratio
$\delta_\varepsilon /\alpha_\varepsilon$ is bounded from above, where
$\delta_\varepsilon = \max\{ \delta_\Omega, \delta_{T,\varepsilon}\}$.
Note that in our definitions of the minimum gains and maximum effects
we used the $\ell_\infty$- and the spectral norm, which are dual to
the $\ell_1$- and the nuclear norm, respectively.  Ultimately, we want
to express the stability assumption in the $\ganorm{\cdot}$-norm which
is the dual norm to the regularization term in
Problem~\ref{prob_SL}. For that we need to compare the $\ell_\infty$-
and the spectral norm. This can be accomplished by using norm
compatibility constants that are given as the smallest possible
$\xi(T(L))$ and $\mu(\Omega(S))$ such that
\[
\|M\|_{\infty} \leq \xi(T(L)) \|M\|\textrm{ for all } M\in T(L),
\,\textrm{ and }\,\|N\| \leq \mu(\Omega(S))
\|N\|_{\infty}\textrm{ for all } N\in \Omega(S),
\]
where $\Omega(S)$ and $T(L)$ are the tangent spaces at points $S$ and
$L$ from the sparse matrix variety $\S(|\supp S|)$ and the low-rank
matrix variety $\Lc(\rank L)$, respectively. Let us now
specify our assumptions in terms of the stability constants from
above.

\begin{assumption}[Stability] \label{a:stability}
  We set $\varepsilon = \xi(T)/2$ and assume that
  \begin{enumerate}
  \item $\alpha = \alpha_{\xi(T)/2} >0$, and
  \item there exists $\nu\in(0,\frac{1}{2}]$ such that
    $\frac{\delta}{\alpha} \leq 1-2\nu$, where $\delta =
    \delta_{\xi(T)/2}$.
  \end{enumerate}
\end{assumption}

The second assumption is essentially a generalization of the
well-known irrepresentability condition, see for example
\cite{ravikumar2011high}. The next assumption ensures that there are
values of $\gamma$ for which stability can be expressed in terms of
the $\ganorm{\cdot}$-norm, that is, a coupled version of stability.

\paragraph{$\gamma$-feasibility assumption.}

The norm compatibility constants $\mu(\Omega)$ and $\xi(T)$ allow
further insights into the realm of problems for which consistent
recovery is possible. First, it can be shown,
see~\cite{chandrasekaran2011rank}, that $\mu(\Omega)\leq
\deg_{\max}(S^\star)$, where $\deg_{\max}(S^\star)$ is the maximum
number of non-zero entries per row/column of $S^\star$, that is,
$\mu(\Omega)$ constitutes a lower bound for
$\deg_{\max}(S^\star)$. Intuitively, if $\deg_{\max}(S^\star)$ is
large, then the non-zero entries of the sparse matrix $S^\star$ could
be concentrated in just a few rows/columns and thus $S^\star$ would be
of low rank. Hence, in order not to confuse $S^\star$ with a low-rank
matrix we want the lower bound $\mu(\Omega)$ on the maximum degree
$\deg_{\max}(S^\star)$ to be small.
\medskip

Second, $\xi(T)$ constitutes a lower bound on the \emph{incoherence}
of the matrix $L^\star$. Incoherence measures how well a subspace is
aligned with the standard coordinate axes. Formally, the incoherence
of a subspace $U\subset \IR^d$ is defined as $\coh(U)=\max_i \|P_U
e_i\|$ where the $e_i$ are the standard basis vectors of $\IR^d$.  It
is known, see again~\cite{chandrasekaran2011rank}, that
\begin{align*}
  \xi(T)=\xi(T(L^\star)) \leq 2\coh(L^\star),
\end{align*}
where $\coh(L^\star)$ is the incoherence of the subspace spanned by
the rows/columns of the symmetric matrix $L^\star$. A large value
$\coh(L^\star)$ means that the row/column space of $L^\star$ is
well aligned with the standard coordinate axes. In this case, the
entries of $L^\star$ do not need to be spread out and thus $L^\star$
could have many zero entries, that is, it could be a sparse
matrix. Hence, in order not to confuse $L^\star$ with a sparse matrix
we want the lower bound $\xi(T)/2$ on the incoherence $\coh(L^\star)$,
or equivalently $\xi(T)$, to be small.
\medskip

Altogether, we want both $\mu(\Omega)$ and $\xi(T)$ to be small to
avoid confusion of the sparse and the low-rank parts. Now, in
Problem~\ref{prob_SL}, the parameter $\gamma >0$ controls the
trade-off between the regularization term that promotes sparsity, that
is, the $\ell_1$-norm term, and the regularization term that promotes
low rank, that is, the nuclear norm term. It turns out that the range
of values for $\gamma$ that are feasible for our consistency analysis
becomes larger if $\mu(\Omega)$ and $\xi(T)$ are small. Indeed, the
following assumption ensures that the range of values of $\gamma$ that
are feasible for our consistency analysis is non-empty.
\begin{assumption}[$\gamma$-feasibility] \label{a:gamma}
The range $[\gamma_{\min},\gamma_{\max}]$ with
\[
\gamma_{\min} = \frac{3\beta(2-\nu)\xi(T)}{\nu \alpha}
\quad\textrm{ and }\quad
\gamma_{\max} = \frac{\nu\alpha}{2\beta (2-\nu)\mu(\Omega)}.
\]
is non-empty. Here, we use the additional problem-specific constant
$\beta = \max\{\beta_\Omega,\beta_T\}$ with
\begin{align*}
  \beta_\Omega &= \underset{M\in \Omega,\, \|M\|=1}{\max} \:
  \|H^\star M\|,\quad\textrm{and} \\
  \beta_T &= \underset{\rho(T, T')\leq\frac{\xi(T)}{2}}{\max} \:\:
  \underset{M\in T',\, \|M\|_{\infty}=1}{\max} \: \|H^\star M\|_{\infty}.
\end{align*}

\end{assumption}

The $\gamma$-feasibility assumption is equivalent to
\begin{align*}
  \mu(\Omega) \xi(T) \leq
  \frac{1}{6}\bPr{\frac{\nu\alpha}{\beta(2-\nu)}}^2.
\end{align*} 
Note that this upper bound on the product $\mu(\Omega) \xi(T)$ is
essentially controlled by the product $\nu\alpha$. It is easier to
satisfy when the latter product is large. This is well aligned with
the stability assumption, because in terms of the stability assumption
the good case is that the product $\nu\alpha$ is large, or more
specifically that $\alpha$ is large and $\nu$ is close to $1/2$.

\paragraph{Gap assumption.}

Intuitively, if the smallest-magnitude non-zero entry $s_{\min}$ of
$S^\star$ is too small, then it is difficult to recover the support of
$S^\star$. Similarly, if the smallest non-zero eigenvalue
$\sigma_{\min}$ of $L^\star$ is too small, then it is difficult to
recover the rank of $L^\star$. Hence, we make the following final
assumption.

\begin{assumption}[Gap]
  \label{a:gap} 
  We require that
\[
s_{\min} \geq \frac{C_S \la_n}{\mu(\Omega)} \quad\textrm{ and }\quad
\sigma_{\min}\geq \frac{C_L \la_n}{\xi(T)^2},
\]
where $C_S$ and $C_L$ are problem-specific constants that are
specified more precisely later.
\end{assumption}

Recall that the regularization parameter $\la_n$ controls how strongly
the eigenvalues of the solution $L_n$ and the entries of the solution
$S_n$ are driven to zero. Hence, the required gaps get weaker as the
number of sample points grows, because the parameter $\la_n$ goes to
zero as $n$ goes to infinity.

\subsection{Consistency theorem}
\label{sec:thm}

We state our consistency result using problem-specific
data-independent constants $C_1$, $C_2$ and $C_3$. Their exact
definitions can be found alongside the proof in
Section~\ref{app:constants}. Also note that the norm compatibility
constant $\xi(T)$ is implicitly related to the number of latent
variables $l$. This is because $\xi(T) \leq 2\coh(L^\star)$ as we have
seen above and $\sqrt{l/d}\leq \coh(L^\star)\leq 1$,
see~\cite{chandrasekaran2011rank}. Hence, the smaller $l$, the better
can the upper bound on $\xi(T)$ be. Therefore, we track $\xi(T)$ and
$\mu(\Omega)$ explicitly in our analysis.

\begin{theorem} [Consistency] \label{thm_consistency}
  Let $S^\star\in\textrm{Sym}(d)$ be a sparse and let $0\preceq
  L^\star\in\textrm{Sym}(d)$ be a low-rank matrix. Denote by $\Omega =
  \Omega(S^\star)$ and $T=T(L^\star)$ the tangent spaces at $S^\star$
  and $L^\star$, respectively to the variety of symmetric sparse
  matrices and to the variety of symmetric low-rank matrices. Suppose
  that we observed samples $x^{(1)}, \ldots, x^{(n)}$ drawn from a
  pairwise Ising model with interaction matrix $S^\star + L^\star$
  such that the stability assumption, the $\gamma$-feasibility
  assumption, and the gap assumption hold. Moreover let $\kpa > 0$, and
  assume that for the number of sample points $n$ it holds that
	\[
  n > \frac{C_1 \kpa}{\xi(T)^4}  d\log d,
  \]
	and that the regularization
  parameter $\la_n$ it set as 
  \[
  \la_n = \frac{C_2}{\xi(T)}\sqrt{\frac{\kpa d\log d}{n}}.
  \] 
  Then, it follows with probability at least
  $1-d^{-\kpa}$ that the solution $(S_n, L_n)$ to the convex program
  \ref{prob_SL} is
  \begin{itemize}
  \item[a)] parametrically consistent, that is, $\ganorm{(S_n -
    S^\star, L_n - L^\star)} \leq C_3 \la_n$, and
  \item[b)] algebraically consistent, that is, $S_n$ and $S^\star$ have
    the same support (actually, the signs of corresponding entries
    coincide), and $L_n$ and $L^\star$ have the same ranks.
  \end{itemize}	
\end{theorem}

\subsection{Outline of the proof}
\label{sec:outline}

The proof of Theorem~\ref{thm_consistency} is similar to the one given
in \cite{ChandrasekaranPW12} for latent variable models with observed
Gaussians. More generally, it builds on a version of the
primal-dual-witness proof technique. The proof consists of the
following main steps:
\begin{itemize}
\item[(1)] First, we consider the \emph{correct model set} $\M$ whose
  elements are all parametrically and algebraically consistent under
  the stability, $\gamma$-feasibility, and gap assumptions. Hence, any
  solution $(S_\M, L_\M)$ to our problem, if additionally constrained
  to $\M$, is consistent.
\item[(2)] Second, since the set $\M$ is non-convex, we consider a
  simplified and linearized version $\Y$ of the set $\M$ and show that
  the solution $(S_\Y, L_\Y)$ to the problem constrained to the
  linearized model space $\Y$ is unique and equals $(S_\M,
  L_\M)$. Since it is the same solution, consistency follows from the
  first step.
\item[(3)] Third, we show that the solution $(S_\M, L_\M) = (S_\Y,
  L_\Y)$ also solves Problem~\ref{prob_SL}.  More precisely, we show
  that this solution is \emph{strictly dual feasible} and hence can be
  used as a witness as required for the primal-dual-witness
  technique. This implies that it is also the unique solution, with
  all the consistency properties from the previous steps.
\item[(4)] Finally, we show that the assumptions from
  Theorem~\ref{thm_consistency} entail all those made in the previous
  steps with high probability. Thereby, the proof is concluded.
\end{itemize}

\section{Discussion} 

Our result, that constitutes the first high-dimensional consistency
analysis for sparse + low-rank Ising models, requires slightly more
samples (in the sense of an additional logarithmic factor $\log d$,
and polynomial probability) than were required for consistent recovery
for the sparse + low-rank Gaussian models considered by
\cite{ChandrasekaranPW12}. This is because the strong tail properties
of multivariate Gaussian distributions do not hold for multivariate
Ising distributions. Hence, it is more difficult to bound the sampling
error $\IE[\Phi] - \Phi^n$ of the second-moment matrices, which
results in weaker probabilistic spectral norm bounds of this sampling
error. Under our assumptions, we believe that the sampling complexity,
that is, the number of samples required for consistent recovery of
sparse + low-rank Ising models, cannot be improved. We also provided a
detailed discussion of why all of our assumptions are important.
\medskip

It would be interesting to test for consistency experimentally, but
this is better done using a pseudo-likelihood approach because it
avoids the problem of computing costly normalizations. We believe that
likelihood and pseudo-likelihood behave similarly, but so far only
much weaker guarantees are known for the pseudo-likelihood approach
than the ones that we prove here.

\acks{We gratefully acknowledge financial support from  the German Science Foundation (DFG)
  grant (GI-711/5-1) within the priority program (SPP 1736) Algorithms
  for Big Data.}

\bibliography{references}

\section{Proof of the  Consistency Theorem}
\label{app:proof}

In this section, we prove Theorem~\ref{thm_consistency}.

\subsection{Preliminaries}
\label{app:constants}

Here, we give an overview of basic definitions and constants that are
used throughout the paper.  The constants are also necessary to refine
the problem-specific constants that appear in the assumptions and
claims of Theorem~\ref{thm_consistency}.

\paragraph{Duplication operator.}
Throughout we use the duplication operator
\[
\D:\textrm{Sym}(d) \to \textrm{Sym}(d) \times \textrm{Sym}(d), \;
M\mapsto (M, M).
\]

\paragraph{Norms.}
During the course of the paper we use several matrix norms. For $M\in
\textrm{Sym}(d)$, the $\ell_1$- and the nuclear norm are given by
\[
\|M\|_1 = \sum_{(i,j)\,\in\,[d]\times[d]} |M_{ij}|, 
\quad\textrm{ and }\: \|M\|_\ast = \sum_{i\,\in\,[d]} |\sigma_i|,
\]
where $\sigma_1, \ldots, \sigma_d$ are the singular values of $M$.  Note that for
$M\succeq0$ it holds $\|M\|_\ast= \tr M$. We also use the respective
dual norms. They are the $\ell_\infty$- and spectral norm given by
\[
\|M\|_\infty = \max_{(i,j)\,\in\,[d]\times[d]} |M_{ij}|, \quad\textrm{ and
}\: \|M\| = \max_{\|v\|_2=1} \|Mv\|_2 = \max_{i\,\in\,[d]}|\sigma_i|,
\]
where $\|\cdot\|_2$ is the standard Euclidean norm for vectors. 

\paragraph{Second-moment matrices and norm of the Hessian.}

In the introduction we used $\Phi^n$ and $\IE[\Phi]$ which actually
are the empirical and the population version of the second-moment
matrix, that is, $\Phi^n = \frac{1}{n}\sum_{k=1}^n
x^{(k)}\bPe{x^{(k)}}^\top$ and $\IE[\Phi] = \IE\bPe{XX^\top}$, where the
expectation is taken \wrt the true Ising model distribution with
parameter matrix $S^\star+L^\star$. Note that the gradient of the
negative log-likelihood satisfies $\nabla \ell(S^\star + L^\star) =
\Phi^\star - \Phi^n$, where we denoted $\Phi^\star = \IE[\Phi]$.
Moreover, we denote the Hessian as $H^\star = \nabla^2 \ell(S^\star +
L^\star)$ and its operator norm is given by
 \[\|H^\star\| = \max_{M\,\in\,\textrm{Sym}(d):\, \|M\|=1}\|H^\star M\|.\]

\paragraph{Norm compatibility constant.}

Since we will encounter the following constant several times in the
proof we give it its own symbol
\[
\omega = \max\left\{\frac{\nu \alpha}{3\beta(2-\nu)},1 \right\}.
\]
Later in Section~\ref{app_step1}, we will show that $\omega$ is
essentially a norm compatibility constant between the
$\ganorm{\cdot}$-norm and the spectral norm.

\paragraph{Problem-specific constants.}

Besides the norm compatibility constant we frequently use some
problem-specific constants that we define below. Here, $r_0 >0$ and
$l(r_0)>0$ are defined in Lemma~\ref{l_CgSLcons_boundedremainder}.
\begin{align*}
  c_0 &= 2l(r_0)\omega \max\left\{1, \frac{\nu\alpha}
  {2\beta(2-\nu)}\right\}^2 \\
  c_1 &= \max\left\{ 1, \frac{\nu\alpha}{2\beta(2-\nu)}\right\}^{-1}
  \frac{r_0}{2} \\
  c_2 &= \frac{40}{\alpha} + \frac{1}{\|H^\star\|}\\
  c_3 &= \bPr{\frac{6(2-\nu)}{\nu} + 1}c_2^2 \|H^\star\| \omega\\
  c_4 &= c_2 + \frac{3\alpha c_2^2 (2-\nu)}{16(3-\nu)} \\
  c_5 &= \max\{c_3, c_4\} \\
  c_6 &= \frac{\nu\alpha c_2}{2\beta (2-\nu)}.
\end{align*}
We now refine the statements of Theorem~\ref{thm_consistency}.

\paragraph{Minimum number of samples required (precise).}

We require at least
\[
n > c \kpa d\log d \max\bPc{\|\Phi^\star\|^{-1}, \|\Phi^\star\|
  \frac{\omega^2}{\xi(T)^2} \bPe{\frac{\alpha\nu}{32(3-\nu)}
    \min\left\{\frac{c_1}{2}, \frac{\alpha\nu\xi(T)}{128
      c_0(3-\nu)}\right\}}^{-2}}.
\]
samples for consistent recovery, where $\kpa>0$ is a positive constant
that is used to control the probability with which consistent recovery
is possible.

\paragraph{Choice of $\la_n$ (precise).}

For our consistency analysis we choose the following value for the
trade-off parameter $\la_n$ between the negative log-likelihood and
the regularization terms
\[
\la_n = \frac{6(2-\nu) }{\nu} \sqrt{\frac{c \kpa d\log d
    \|\Phi^\star\|}{n}} \frac{\omega}{\xi(T)}.
\]

\paragraph{Gap assumption (precise).}

The precise gap assumptions on the smallest-magnitude non-zero entry
$s_{\min}$ of $S^\star$ and the smallest non-zero eigenvalue
$\sigma_{\min}$ of $L^\star$ is given by
\[
s_{\min} \geq \frac{c_6 \la_n}{\mu(\Omega)}, \quad\textrm{ and }\quad
\sigma_{\min}\geq \frac{c_5 \la_n}{\xi(T)^2}.
\]

\subsection{Tangent space lemmas}

Low-rank tangent spaces play a fundamental role throughout the proof.
Therefore we characterize the tangent spaces at smooth points of the
low-rank variety before moving on.

\begin{lemma}
  \label{lem:lowranktspace}
  Suppose $L\in \Lc(r)$ is a rank-$r$ matrix. Then, the tangent space
  to $\Lc(r)$ at $L$ is given by
  \begin{align*}
    T(L) = \bPc{U X^\top + X U^\top:\; X\in\IR^{d\times r}} \subset
    \textrm{Sym}(d)
  \end{align*}
  where $L=UD U^\top$ is the (restricted) eigenvalue decomposition of
  $L$, that is, $U\in\IR^{d\times r}$ has orthonormal columns and
  $D\in\IR^{r\times r}$ is diagonal with the eigenvalues on the diagonal.
\end{lemma}
\begin{proof}
  The tangent space at $L$ is given by the span of all tangent vectors
  at zero to smooth curves $\gamma:(-1, 1)\to\Lc(r)$ initialized at
  $L$, that is, $\gamma(0)=L$.  Because $L$ has rank $r$ it is a
  smooth point of $\Lc(r)$ and we can assume that $\gamma(t) = A(t)
  \sign(D) A(t)^\top$ with rank-$r$ matrices $A(t)\in\IR^{d\times r}$
  for all $t\in(-1, 1)$ and $\sign(D)\in\IR^{d\times d}$ is the
  diagonal matrix whose diagonal entries are the signs of the
  eigenvalues of $L$, that is, they are in $\{-1, 1\}$. We can assume
  the signs of the eigenvalues along the curve to be fixed because we
  only consider smooth curves.  In particular $L=A(0) \sign(D)
  A(0)^\top$, so it must hold $A(0) = U |D|^{1/2}$. 
  Now, by the chain rule it holds
  \begin{align*}
    \gamma'(0) &= A(0) \sign(D) A'(0)^\top + A'(0)\sign(D)A(0)^\top \\
    &= U |D|^{1/2} \sign(D) A'(0)^\top + A'(0)\sign(D) |D|^{1/2} U^\top
  \end{align*}
  We still need to show that $A'(0)\sign(D) |D|^{1/2}$ can take
  arbitrary values.  To do so, for any $X\in\IR^{d\times r}$ consider
  $A(t)=A(0)+t X|D|^{-1/2}\sign(D)$ which has rank $r$ for
  sufficiently small $t$ since $A(0)$ has rank $r$ and the curve is
  smooth.  Moreover, it holds $A'(0)= X|D|^{-1/2}\sign(D)$.  Now with
  the particular choice of $A(t)$, since
  \[
  A'(0)\sign(D) |D|^{1/2}= X|D|^{-1/2}\sign(D)\sign(D) |D|^{1/2} =X
  \] 
  the tangential vector of the corresponding curve at zero is $UX^\top
  + X U^\top$.  
\end{proof}

Note that the variety of symmetric low-rank matrices $\Lc(r)$ has
dimension $rd - \frac{r(r-1)}{2}$. Since $L$ is a smooth point in
$\Lc(r)$ the tangent space $T(L)$ has the same dimension.
\medskip

One consequence of the form of the tangent spaces is the following
lemma that concerns the norms of projections on certain tangent spaces
and their orthogonal complements.

\begin{lemma}
  \label{lem:projnorms}
  For any two tangent spaces $\Omega$ and $T$ at any smooth points
  \wrt the varieties $\S(s)$ and $\Lc (r)$, respectively, we can bound
  the norms of projections of matrices $M, N\in\textrm{Sym}(d)$ in the
  following manner:
  \begin{align*}
    \|P_\Omega M\|_\infty &\leq \|M\|_\infty \quad\textrm{ and }\quad
    \|P_{\Omega^\perp}M\|_\infty \leq \|M\|_\infty \\
    \|P_T N\| &\leq 2\|N\| \quad\textrm{ and }\quad
    \|P_{T^\perp}N\| \leq \|N\|.
  \end{align*}
  In particular, for $\Y=\Omega\times T$ we have
  \[
    \ganorm{P_\Y(M,N)} \leq 2\ganorm{(M, N)}\quad\textrm{ and }\quad
    \ganorm{P_{\Y^\perp}(M,N)} \leq \ganorm{(M, N)}.
  \]
\end{lemma}
\begin{proof}
  Recall that from Lemma~\ref{lem:lowranktspace} we have for smooth
  points $L\in \Lc(r)$ that
  \begin{align*}
    T = T(L) = \bPc{U X^\top + X U^\top:\; X\in\IR^{d\times r}}
    \subset \textrm{Sym}(d)
  \end{align*}
  where $L=UD U^\top$ is the (restricted) singular value decomposition
  of $L$. Then, we have more explicitly that
  \begin{align*}
	P_T N= \P_U N + N \P_U - \P_U N \P_U = \P_U N + (I - \P_U) N \P_U
  \end{align*}
  where $\P_U= UU^\top = \P_U^\top$ projects onto the column space of
  $U$. Note that $P_T N \in T$ since
  \begin{align*}
    \P_U N + N \P_U - \P_U N \P_U
    &= \P_U \left(N-\frac{1}{2}N\P_U\right) +
    \left(N^\top -\frac{1}{2}\P_UN^\top \right) \P_U \\
    &= U \left(U^\top \left(N-\frac{1}{2}N\P_U\right)\right)
    + \left(\left(N^\top -\frac{1}{2}\P_U N^\top \right) U\right)U^\top, 
  \end{align*}
  and that $N-P_T N$ is orthogonal to $T$ since
  \[
    N-P_T N = N -\P_U N -N\P_U +\P_U N\P_U = (I-\P_U) N (I-\P_U)
  \]
  and since for any $UX^\top +XU^\top\in T$ we have
  \begin{align*}
    &\scalp{(I-\P_U) N (I-\P_U), UX^\top +XU^\top} \\
    &\qquad\qquad\qquad= \tr \bPr{(I-\P_U) N (I-\P_U) UX^\top}  +
    \tr \bPr{(I-\P_U) N (I-\P_U)XU^\top} \\
    &\qquad\qquad\qquad= \tr \bPr{(I-\P_U) N (I-\P_U)XU^\top} \\
    &\qquad\qquad\qquad= \tr \bPr{XU^\top (I-\P_U) N (I-\P_U)} \\
    &\qquad\qquad\qquad= \tr \bPr{ \bPr{(I-\P_U)UX^\top}^\top N (I-\P_U)} \\
    &\qquad\qquad\qquad=0
  \end{align*}
  where the second and last inequality follow from $\P_U UX^\top = U
  U^\top U X^\top = UX^\top = I U X^\top$, and we used $\tr (AB) =
  \tr(BA)$ in the third equality. Thus, $P_T N$ is indeed the
  orthogonal projection of $N$ onto $T$.  Now, by sub-multiplicativity
  of the spectral norm
  \begin{align*}
    \|P_T N\| &\leq  \|\P_U N\| + \|(I - \P_U) N \P_U\| \\
    &\leq \|\P_U\| \| N\| + \|(I - \P_U)\| \|N\| \|\P_U\|
    \leq 2 \|N\|
  \end{align*}
  since $\P_U$ and $I-\P_U$ are projection matrices, that is,
  $\|\P_U\|\leq 1$ and $\|I-\P_U\|\leq1$. The other claims are easy or
  follow similarly.
\end{proof}

\subsection{Coupled stability}
\label{app:coupledstability}

Provided that the stability and the $\gamma$-feasibility assumptions
hold, here we will prove fundamental bounds in the
$\ganorm{\cdot}$-norm on the minimum gains and the maximum effects of
the Hessian $H^\star = \nabla^2 \ell(S^\star + L^\star)$ on the direct
sum of the tangent space $\Omega$ and tangent spaces $T'$ close to the
true tangent space $T=T(L^\star)$.  These bounds are analogous to the
stability assumption, only that now they take the necessary coupling
between the different tangent spaces into account. Therefore, we refer
to the result as \emph{coupled stability}.  Importantly, this coupled
stability will guarantee transversality of the respective tangent
spaces.
\medskip

The proof will use the following auxiliary lemma that bounds 
the norm compatibility constant $\xi$ on a low-rank tangent space by
the one on a nearby low-rank tangent space.
\begin{lemma} \label{lem_boundtwisted}
  Let $T_1, T_2$ be two matrix subspaces of the same dimension with
  bounded twisting, that is,
\begin{align*}
\rho(T_1, T_2) = \underset{\|M\| =1}{\max}
\bNorm {\bPe{P_{T_1}-P_{T_2}}(M)]} < 1.
\end{align*}
Then, it holds that
  \begin{align*}
    \xi(T_2)\leq\frac{1}{1-\rho(T_1, T_2)} \bPe{\xi(T_1) +
      \rho(T_1, T_2)}.
  \end{align*}
\end{lemma}
\begin{proof}
  See \cite[Lemma~3.1]{ChandrasekaranPW12}.
\end{proof}

Now, we are ready to formulate the bounds. We do so in terms of the
problem-specific constants $\alpha,\beta,\nu$. Moreover, we use the
norm compatibility constants $\mu(\Omega)$ and $\xi(T)$.

\begin{proposition}[Coupled stability]
  \label{prop_irreptransvcond}
  Suppose that the stability and the gamma-feasibility assumptions
  hold. Let $\gamma \in [\gamma_{\min},\gamma_{\max}]$. Then, for
  $\Y=\Omega\times T'$ with $\rho(T, T')\leq\frac{\xi(T)}{2 }$ it
  holds
  \begin{itemize}
    \item[(a)] The minimum gain of $H^\star$ restricted to
      $\Omega\oplus T'$ is bounded from below, that is, for all $(M,N)\in
      \Y$ it holds that
      \begin{align*}
        \ganorm{P_{\Y} \D H^\star (M + N) } \geq
        \frac{\alpha}{2} \ganorm{(M, N)}.
    \end{align*}
    \item[(b)] The maximum effect of elements in $\Y=\Omega\times T'$ on the
      orthogonal complement $\Y^\perp$ is bounded from above, that is, it
      holds for all $(M,N)\in \Y$ that
      \begin{align*}
        \ganorm{P_{\Y^\perp} \D H^\star (M + N)} \leq (1-\nu)
        \ganorm{P_{\Y} \D H^\star (M + N)}.
      \end{align*}
  \end{itemize}
\end{proposition}
\begin{proof}
  Note first that the range for $\gamma$ is non-empty because of the
  $\gamma$-feasibility assumption.
  \medskip
  
  (a) For the first claim note that
  \begin{align*}
    P_{\Y} \D H^\star (M + N) = (P_\Omega H^\star (M+N), P_{T'}
    H^\star(M+N)).
  \end{align*}
  For bounding the $\ganorm{\cdot}$-norm of the tuple on the right
  hand side, we bound the respective norms for both tuple entries
  separately. For the first entry we get that
  \begin{align*}
    \| P_\Omega H^\star(M+N)\|_\infty
    &\geq \| P_\Omega H^\star M\|_\infty - \| P_\Omega H^\star N\|_{\infty} \\
    &\geq \| P_\Omega H^\star M\|_\infty - \| H^\star N\|_{\infty} \\
    &\geq \alpha_\Omega \|M\|_\infty - \beta_T \|N\|_\infty \\
    &\geq \alpha  \|M\|_\infty - \beta \|N\|_\infty \\
    &\geq \alpha  \|M\|_\infty - \beta \xi(T') \|N\| \\
    &\geq \alpha  \|M\|_\infty - 3\beta \xi(T) \|N\|,
  \end{align*}
  where the last inequality follows from 
  \[
    \xi(T') \leq \frac{\xi(T)+\rho(T, T')}{1-\rho(T, T')} \leq
    \frac{\xi(T) + \xi(T)/2}{1-\xi(T)/2} = \frac{3 \xi(T)}{2-\xi(T)}
    \leq 3 \xi(T),
  \]
  which itself follows from Lemma~\ref{lem_boundtwisted} and from
  $\rho(T, T')\leq \frac{\xi(T)}{2}\leq \frac{1}{2} <1$. Note that
  since $\|\cdot\|_\infty \leq \|\cdot\|$, we have $\xi(T)\leq 1$.
 \medskip

  Similarly, we get for the second entry in the tuple that
  \begin{align*}
    \| P_{T'} H^\star(M+N)\|
    &\geq \| P_{T'} H^\star N\| - \| P_{T'} H^\star M\| \\
    &\geq \| P_{T'} H^\star N\| - 2\| H^\star M\| \\
    &\geq \alpha_{T,\xi(T)/2} \|N\|- 2\beta_\Omega \|M\| \\
    &\geq \alpha \|N\|- 2\beta \|M\| \\
    &\geq \alpha \|N\|- 2\beta \mu(\Omega)\|M\|_\infty, 
  \end{align*}
  where the second inequality follows from Lemma~\ref{lem:projnorms}.
  \medskip
	
  Putting both bounds together gives
  \begin{align}
    \ganorm{P_{\Y} \D H^\star (M + N)}
    &\geq \max\bPc{\frac{\alpha\|M\|_\infty - 3\beta \xi(T)
        \|N\|}{\gamma}, \alpha \|N\|- 2\beta
      \mu(\Omega)\|M\|_\infty} \notag \\
    &\geq \alpha \ganorm{(M, N)} - \beta \max\bPc{\frac{3\xi(T)\|N\|}{\ga},
      2\mu(\Omega)\|M\|_\infty} \notag \\
    &\geq \alpha \ganorm{(M, N)} - \beta \max\bPc{\frac{3\xi(T)}{\ga},
      2\mu(\Omega)\gamma } \ganorm{(M, N)} \notag \\ 
    &\geq \alpha \ganorm{(M, N)} - \frac{\nu \alpha}{(2-\nu)} \ganorm{(M, N)}
    \label{eq:gammaf} \\
    &\geq \frac{\alpha}{2} \ganorm{(M, N)}, \notag
  \end{align}
  where the second inequality is triangle inequality, and the
  second-to-last inequality follows from the bounds on $\gamma$, and
  the final inequality follows from $\nu\leq 1/2$.
  \medskip

  (b) For the second claim note that
  \begin{align*}
    P_{\Y^\perp} \D H^\star (M + N) = (P_{\Omega^\perp} H^\star
    (M + N), P_{{T'}^\perp} H^\star(M + N)).
  \end{align*}
  Again, we bound the respective norms of both tuple entries. For the
  first entry we get that
  \begin{align*}
    \| P_{\Omega^\perp} H^\star(M+N)\|_{\infty}
    &\leq \|P_{\Omega^\perp} H^\star M\|_{\infty} + \| P_{\Omega^\perp}
    H^\star N\|_{\infty} \\
    &\leq \|P_{\Omega^\perp} H^\star M\|_{\infty} + \| H^\star N\|_{\infty} \\
    &\leq \delta_\Omega \|M\|_\infty + \beta_T \|N\|_{\infty} \\
    &\leq \delta \|M\|_\infty + \beta \xi(T') \|N\| \\
    &\leq \delta \|M\|_\infty + 3\beta \xi(T) \|N\|.
  \end{align*}
  Similarly, we get for the second entry that
  \begin{align*}
    \| P_{T'^\perp} H^\star(M+N)\| &\leq \| P_{T'^\perp}
    H^\star M\| + \| P_{T'^\perp} H^\star N\| \\
    &\leq \| H^\star M\| + \| P_{T'^\perp} H^\star N\|_{{2}}\\
    &\leq \beta_\Omega  \|M\| + \delta_{T,\xi(T)/2} \|N\|\\
    &\leq \beta \mu(\Omega)\|M\|_\infty + \delta \|N\| \\
    &\leq 2\beta \mu(\Omega)\|M\|_\infty + \delta \|N\|.
  \end{align*}
  Together it follows
  \begin{align*}
    \ganorm{P_{\Y^\perp}\D H^\star (M + N)} &\leq \delta  \ganorm{(M, N)}+
    \beta \max\bPc{\frac{3 \xi(T) \|N\|}{\gamma},
      2\mu(\Omega)\|M\|_\infty} \\
    &\leq \delta  \ganorm{(M, N)}+ \beta \max\bPc{\frac{3\xi(T)}{\ga},
      2\mu(\Omega)\ga}\ganorm{(M, N)}\\
    &\leq \bPr{\delta +\frac{\nu\alpha}{2-\nu}} \ganorm{(M, N)} \\
    &\leq \bPr{\delta +\frac{\nu\alpha}{2-\nu}} \bPr{ \alpha -
      \frac{\nu \alpha}{(2-\nu)}}^{-1} \ganorm{P_{\Y} \D H^\star
      (M + N)}\\
    &\leq (1-\nu)\ganorm{P_{\Y} \D H^\star  (M + N)},
  \end{align*}
  where we once again used the bounds on $\gamma$ for the third
  inequality, the fourth inequality is the same as in
  Equation~\eqref{eq:gammaf}, and the last inequality follows from
  the stability assumption, that is, $\frac{\delta}{\alpha} \leq 1-2\nu$
  and some algebra, namely
  \begin{align*}
    \delta +\frac{\nu\alpha}{2-\nu} \leq(1-2\nu)\alpha
    +\frac{\nu\alpha} {2-\nu} = (1-\nu)\alpha +\frac{\nu\alpha -
      (2-\nu)\nu\alpha}{2-\nu} =(1-\nu)\bPr{ \alpha - \frac{\nu
        \alpha}{(2-\nu)}}.
  \end{align*}
\end{proof}

It is easy to see that coupled stability implies transversality.

\begin{remark}[transversality]
  Suppose there exists $0\neq K\in \Omega\,\cap\, T'$. Then, choosing
  $M=K$ and $N=-K$ in Proposition~\ref{prop_irreptransvcond}(a)
  contradicts the stability assumption because
  \[
  0 = \ganorm{P_{\Y} \D H^\star (K + (-K))} \geq \frac{\alpha}{2}
  \ganorm{(K,-K)} = \frac{\alpha}{2}\max\bPc{
    \frac{\|K\|_\infty}{\gamma}, \|K\| } > 0
  \]
  since $\alpha>0$. Hence, we have transversality, that is,
  $\Omega\cap T'=\{0\}$.
\end{remark}

From now on, we generally assume that the stability assumption and the
$\gamma$-feasibility assumption are satisfied (and hence it holds
coupled stability for nearby tangent spaces).

\subsection{Step 1: Constraining the problem to consistency}
\label{app_step1}

In this section, we consider a restricted problem where parametric and
algebraic consistency of the solution are explicitly enforced, namely
\begin{equation}
\begin{array}{lrlr}
  (S_\M, L_\M) = &\underset{S, \,L} {\argmin}
  &\ell(S+L) + \la_n \bPr{\gamma \|S\|_{1} + \|L\|_\ast} \\
  &\subjto& (S, L)\in \M&
\end{array} \tag{\mbox{SL-$\M$}}\label{prob_variety}
\end{equation}
with the \emph{non-convex} constraint set
(compare~\cite{ChandrasekaranPW12})
\begin{align*}
  \M = &\Big\{(S, L) :\, S\in\Omega(S^\star),\, \rank(L)\leq \rank(L^\star),\\
  &\quad \ganorm{\D H^\star (\Delta_S + \Delta_L)} \leq 9\la_n,\,
  \|P_{T^\perp}(\Delta_L)\| \leq \frac{\xi(T) \la_n}
  {\omega \|H^ \star\|} \Big\}, 
\end{align*}
where $\Delta_S = S-S^\star$ and $\Delta_L = L - L^\star$ are the
errors, and the constant $\omega$ and the operator norm $\|H^\star\|$
of the Hessian are defined in Section~\ref{app:constants}.  The first
two constraints ensure that $S$ and $L$ are in the correct algebraic
varieties, and the latter two constraints enforce parametric
consistency.  Later in the proof, we will show that the
additional constraints are actually inactive at the optimal solution
$(S_\M, L_\M)$. 

\subsubsection{Parametric consistency}

Let us first discuss how the last two constraints in the description
of $\M$ enforce parametric consistency.  For that we need the
following lemma that shows that $\omega$ is closely related to a norm
compatibility constant between the $\ganorm{\cdot}$- and the spectral
norm.
\begin{lemma}
  \label{lem:normcompg2}
  For $\gamma \in [\gamma_{\min},\gamma_{\max}]$ and
  $M\in\textrm{Sym}(d)$ we have
  \[
  \ganorm{\D M} \leq \frac{\omega}{\xi(T)}\|M\|.
  \]
\end{lemma}
\begin{proof}
  We have $\|M\|_\infty \leq \|M\|$ for $M\in \textrm{Sym}(d)$,
  and it holds $\xi(T)\leq 1$. By our choice of $\gamma$ it follows that
  \begin{align*}
    \ganorm{\D M}
    &= \max\left\{\frac{\|M\|_{\infty}}{\gamma},
    \|M\|\right\} \\
    &\leq \max\left\{\frac{1}{\gamma}, 1\right\} \|M\|\\
    &\leq \max\left\{\frac{1}{\gamma_{\min}},1\right\}
    \|M\| = \max\left\{\frac{\nu \alpha}{3\beta(2-\nu)\xi(T)},1\right\}
    \|M\|\\
    &\leq \max\left\{\frac{\nu \alpha}{3\beta(2-\nu)\xi(T)},
    \frac{1}{\xi(T)}\right\}\|M\| \\
    &=\frac{\omega}{\xi(T)} \|M\|,
  \end{align*}
  using $\omega = \max\left\{\frac{\nu \alpha}{3\beta(2-\nu)},1 \right\}$.
\end{proof}  

\begin{proposition}[parametric consistency]
  \label{prop:enforceparacons}
  Let $\gamma \in [\gamma_{\min},\gamma_{\max}]$ and let
  $(S,L)\in\M$. Then, with $c_2=\frac{40}{\alpha} +
  \frac{1}{\|H^\star\|}$ it holds that
  \[
  \ganorm{(\Delta_S, \Delta_L)}\leq c_2\la_n.
  \]
  
\end{proposition}
\begin{proof}
  Let $\Y=\Omega\times T$. We have
  \begin{align*}
  (\Delta_S, \Delta_L) &= P_\Y(\Delta_S, \Delta_L) + P_{\Y^\perp}
    (\Delta_S, \Delta_L)\\
    &= P_\Y(\Delta_S, \Delta_L) +
    (P_{\Omega^\perp} (\Delta_S),P_{T^\perp}(\Delta_L)) 
    = P_\Y(\Delta_S, \Delta_L) + (0,P_{T^\perp}(\Delta_L)), 
  \end{align*}
  since $\Delta_S\in\Omega$, and by triangle inequality 
  \[
  \ganorm{(\Delta_S, \Delta_L)} \leq \ganorm{P_\Y(\Delta_S, \Delta_L)}
  + \ganorm{(0, P_{T^\perp}(\Delta_L))} = \ganorm{P_\Y(\Delta_S, \Delta_L)}
  + \|P_{T^\perp}(\Delta_L)\|.
  \]

  The component in the $\Y^\perp$ direction can be bounded as
  \[
  \|P_{T^\perp}(\Delta_L)\| \leq \frac{\xi(T) \la_n}{\omega
    \|H^\star\|} \leq \frac{\la_n}{\|H^\star\|}
  \]
  by the fourth constraint in the definition of $\M$, $\xi(T)\leq 1$,
  and that by definition $\omega \geq 1$.

  The component in the $\Y$ direction can be bounded as
  \begin{align*}
    \ganorm{P_\Y(\Delta_S, \Delta_L))}
    &\leq \frac{2}{\alpha} \ganorm{P_\Y\D H^\star
      \big(P_{\Y} (\Delta_S), P_T (\Delta_L)\big)}\\
    &\leq \frac{4}{\alpha} \ganorm{\D H^\star
       \big(P_{\Y} (\Delta_S), P_T (\Delta_L)\big)} \\
    &\leq \frac{4}{\alpha} \left(
    \ganorm{\D H^\star (\Delta_S +\Delta_L)}
    + \ganorm{\D H^\star P_{T^\perp}(\Delta_L)}\right) \\
    &\leq  \frac{4}{\alpha} \left( 9\la_n
    + \ganorm{\D H^\star P_{T^\perp}(\Delta_L)}\right) \\
    &\leq \frac{4}{\alpha} (9\la_n + \la_n) \,=\, \frac{40\la_n}{\alpha},
  \end{align*}
  where the first inequality is implied by
  Proposition~\ref{prop_irreptransvcond}(a), the second inequality
  follows from Lemma~\ref{lem:projnorms}, the third inequality is
  another application of the triangle inequality since we have shown
  above that
  \[
  P_{\Y}(\Delta_S,\Delta_L) = (\Delta_S, \Delta_L)
  - (0,P_{T^\perp}(\Delta_L)),
  \]
  the fourth inequality is the third constraint in the definition of
  $\M$, and the last inequality follows from
  \[
    \ganorm{\D H^\star P_{T^\perp}(\Delta_L)}
    \leq \frac{\omega}{\xi(T)} \|H^\star P_{T^\perp}(\Delta_L)\|
    \leq \frac{\omega}{\xi(T)} \|H^\star\| \|P_{T^\perp}(\Delta_L)\|
    \leq \la_n,
  \]
  where the first inequality follows from Lemma~\ref{lem:normcompg2},
  and the last inequality is the fourth constraint in the definition
  of $\M$.
  \medskip
  
  The claimed bound on $\ganorm{(\Delta_S, \Delta_L)}$ follows by
  putting together the bounds of the components in the directions of
  $\Y^\perp$ and $\Y$.
\end{proof}

Note that Proposition~\ref{prop:enforceparacons} implies parametric
consistency since we assume that $\la_n$ goes to zero as $n$ goes to
infinity. 

\subsubsection{Algebraic consistency}

For obtaining algebraic consistency we now also require that the gap
assumption is satisfied.

\begin{proposition}[algebraic consistency]
  \label{prop:aconsistM}
  Under the gap assumption (and the stability and $\gamma$-feasibility
  assumptions) we have for $(S,L)\in \M$ that
  \begin{enumerate}
  \item[(a)] $S$ has the same support as $S^\star$. Actually, it holds
    the stronger sign consistency, that is, the corresponding non-zero
    entries in $S$ and $S^\star$ have the same sign.
  \item[(b)] $L$ has the same rank as $L^\star$. Hence, $L$ is a
    smooth point of the variety of matrices with rank not greater than
    $\rank(L^\star)$. Furthermore, $L$ is positive semidefinite
    although this has not been required in the definition of $\M$.
  \end{enumerate}
\end{proposition}
\begin{proof}
  (a) The matrix $S$ has the same support as $S^\star$ since
  $S\in\Omega=\Omega(S^\star)$ and
  \[
  \|S-S^\star\|_{\infty} = \|\Delta_S\|_{\infty} \leq \gamma
  \ganorm{(\Delta_S, \Delta_L)} \leq \gamma c_2\la_n \leq
  \frac{\nu\alpha}{2\beta (2-\nu)\mu(\Omega)} c_2\la_n <
  \frac{c_6}{\mu(\Omega)}\la_n\leq s_{\min},
  \]
  where the first inequality holds just by the definition of the
  $\ganorm{\cdot}$-norm, the second inequality holds by
  Proposition~\ref{prop:enforceparacons}, the third inequality is
  implied by $\gamma$-feasibility, and the last inequality follows
  from the gap assumption.
  \medskip

  (b) We have
  \begin{align*}
    \sigma_{\min} &\geq \frac{c_3 \la_n}{\xi(T)^2}
    = \bPr{\frac{6(2-\nu)}{\nu}+1} c_2^2 \|H^\star\| \omega
    \frac{\la_n}{\xi(T)^2} \\
    &\geq 19c_2^2 \|H^\star\| \omega \frac{\la_n}{\xi(T)^2}
    \geq  19c_2\la_n \frac{\omega}{\xi(T)^2}
    \geq 19c_2\la_n \geq 19\|\Delta_L\|,
  \end{align*}
  where the first inequality follows from the gap assumption, the
  second inequality follows from $\nu\leq \frac{1}{2}$, the third
  inequality follows from $c_2\geq \|H^\star\|^{-1}$, the fourth
  inequality follows from $\xi(T)\leq 1$ and $\omega \geq 1$, and the
  last inequality follows from Proposition~\ref{prop:enforceparacons}.
  Now, note that $\|\Delta_L\|$ is the largest eigenvalue of
  $\Delta_L$. Hence, $L=L^\star + \Delta_L$ must have the same rank as
  $L^\star$ and the positive semidefiniteness of $L$ follows from the
  one of $L^\star$.
\end{proof}

\subsubsection{Further properties}

We can draw some further conclusions from the properties of the
elements in $\M$. Amongst them we find that the fourth constraint in the
definition of $\M$ is non-binding. For drawing the conclusions we need
the following Lemma.

\begin{lemma}
  \label{lem:spectralcontrol}
  Let $M\in\IR^{d\times d}$ be a rank-$r$ matrix with smallest
  non-zero singular value $\sigma$, and let $\Delta$ be a perturbation
  to $M$ such that $\|\Delta\|\leq \frac{\sigma}{8}$ and
  $M+\Delta$ is still a rank-$r$ matrix.  Then,
  \begin{enumerate}
  \item[(a)] the twisting between the two low-rank variety tangent
    spaces can be controlled via
    \[
    \rho(T(M+\Delta), T(M)) \leq \frac{2\|\Delta\|}{\sigma},
    \textrm{ and}
    \]
  \item[(b)] the component of the perturbation in the normal direction
    can be bounded by
    \[
    \|P_{T(M)^\perp}(\Delta)\| \leq\frac{\|\Delta\|^2}{\sigma}.
    \]
  \end{enumerate}
\end{lemma}
\begin{proof}
  See Propositions 2.1 and 2.2 in~\cite{ChandrasekaranPW12}.
\end{proof}

\begin{corollary}
  \label{cor:nonbinding4}
  Under the stability, $\gamma$-feasibility, and gap assumptions we have
  that the fourth constraint in the definition of $\M$ is non-binding,
  or more precisely
  \[
  \|P_{T^\perp}(\Delta_L)\| \leq
  \frac{\xi(T) \la_n}{19\omega \|H^\star\|}.
  \]
\end{corollary}
\begin{proof}
  The proof of Proposition~\ref{prop:aconsistM}(b) shows that
  $\|\Delta_L\|\leq \frac{\sigma_{\min}}{19} \leq
  \frac{\sigma_{\min}}{8}$ and that $L=L^\star + \Delta_L$ has the
  same rank as $L^\star$. Hence, we can use
  Lemma~\ref{lem:spectralcontrol}(b) and get
  \[
  \|P_{T^\perp}(\Delta_L)\|
  \leq\frac{\|\Delta_L\|^2}{\sigma_{\min}}
  \leq \frac{c_2^2 \la_n^2}{\sigma_{\min}}
  \leq \frac{c_2^2 \xi(T)^2 \la_n}{c_3} \\
  \leq \frac{\xi(T)^2 \la_n}{19\omega \|H^\star\|}
  \leq \frac{\xi(T) \la_n}{19\omega \|H^\star\|},
  \]
  where the second inequality follows
  from Proposition~\ref{prop:enforceparacons}, the third inequality
  follows from the gap assumption, the fourth inequality
  follows from the definition of $c_3$ and $\nu \leq \frac{1}{2}$, and
  the last inequality follows from $\xi(T)\leq 1$. Observe that
  \[
  \frac{\xi(T) \la_n}{19\omega \|H^\star\|} < \frac{\xi(T)
    \la_n} {\omega \|H^\star\|},
  \]
  and thus the fourth constraint in the definition of $\M$ is
  non-binding.
\end{proof}

We collect further properties of elements in $\M$ that we will use
later on in the following corollary.

\begin{corollary}
  \label{cor:mconsconclus}
  Under the stability, $\gamma$-feasibility, and gap assumptions we
  have that
  \begin{enumerate}
  \item[(a)] $\rho(T, T(L))\leq \frac{\xi(T)}{4}$,
  \item[(b)] $\ganorm{\D H^\star P_{T(L)^\perp}(L^\star)}\leq
    \frac{\la_n \nu}{6(2-\nu)}$, and
  \item[(c)] $\|P_{T(L)^\perp}(L^\star)\|  \leq \frac{16(3-\nu)\la_n}
    {3\alpha(2-\nu)}$.
  \end{enumerate}
\end{corollary}
\begin{proof}
  (a) In the proof of Corollary~\ref{cor:nonbinding4} we have seen
  that $\|\Delta_L\| \leq \frac{\sigma_{\min}}{8}$. Therefore, we can
  apply Lemma~\ref{lem:spectralcontrol}(a) such that
  \[
  \rho(T, T(L)) \leq \frac{2\|\Delta_L\|}{\sigma_{\min}}
  \leq \frac{2C_2 \xi(T)^2}{c_3}
  \leq \frac{2 \xi(T)^2}{19\omega c_2 \|H^\star\|}
  \leq \frac{2 \xi(T)}{19}
  \leq \frac{\xi(T)}{4},
  \]
  where the second inequality follows
  from the gap assumption, the third inequality follows from
  the definition of $c_3$ and $\nu \leq \frac{1}{2}$, and the fourth
  inequality follows from $c_2\geq \|H^\star\|^{-1}$, that is,
  $c_2\|H^\star\| \geq 1$.
  \medskip
  
  (b) Let $\sigma'$ denote the minimum non-zero singular value of $L$.
  From the proof of Proposition~\ref{prop:aconsistM}(b) we have
  $\sigma_{\min} \geq 19 \|\Delta_L\|$ and thus
  \[
  \sigma' \geq \sigma_{\min}- \|\Delta_L\| \geq
  19\|\Delta_L\|-\|\Delta_L\| = 18\|\Delta_L\|.
  \]
  Hence, we can apply Lemma~\ref{lem:spectralcontrol}(b), where we
  consider $L^\star$ as a perturbation of $L$, and get
  \begin{align*}
    \|P_{T(L)^\perp}(L^\star)\| = \|P_{T(L)^\perp}(L-\Delta_L)\|
    = \|P_{T(L)^\perp}(\Delta_L)\| 
    \leq \frac{\|\Delta_L\|^2}{\sigma'}
    \leq \frac{c_2^2\la_n^2}{\sigma'}
    \leq \frac{\nu\xi(T)\la_n}{6(2-\nu)\omega \|H^\star\|},
  \end{align*}
  where the second inequality follows from
  Proposition~\ref{prop:enforceparacons}, and the last inequality
  follows from
  \begin{align*}
    \sigma' &\geq \sigma_{\min}- \|\Delta_L\|
    \geq \bPr{\frac{c_3}{\xi(T)^2} - c_2} \la_n
    = \bPr{\bPr{\frac{6(2-\nu)}{\nu} + 1}
      \frac{c_2^2 \omega \|H^\star\|}{\xi(T)^2} - c_2} \la_n \\
    &\geq \frac{6(2-\nu)}{\nu}\frac{c_2^2 \omega \|H^\star\|}{\xi(T)}
    \la_n + \bPr{\frac{c_2^2 \omega \|H^\star\|}{\xi(T)} - c_2 }\la_n 
    \geq \frac{6(2-\nu)}{\nu}\frac{c_2^2 \omega \|H^\star\|}
       {\xi(T)}\la_n,
  \end{align*}
  where the second inequality follows from the gap assumption and
  Proposition~\ref{prop:enforceparacons}, the equality follows from
  the definition of $c_3$, the third inequality follows from
  $\xi(T)\leq 1$, and the last inequality follows from $c_2\geq
  \|H^\star\|^{-1} \geq \|H^\star\|^{-1} \omega^{-1} \xi(T)$ using
  $\omega\geq1$.
  \medskip
  
  Hence, we have
  \[
  \ganorm{\D H^\star P_{T(L)^\perp}(L^\star)}
  \leq \frac{\omega}{\xi(T)} \|H^\star P_{T(L)^\perp}(L^\star)\|
  \leq \frac{\omega \|H^\star\|}{\xi(T)}\|P_{T(L)^\perp}(L^\star)\|
  \leq \frac{\nu\la_n}{6(2-\nu)},
  \]
  where the first inequality follows from Lemma~\ref{lem:normcompg2},
  the second inequality follows from the definition of $\|H^\star\|$,
  and the last inequality follows from the upper bound on
  $\|P_{T(L)^\perp}(L^\star)\|$ that we have just derived above.
  \medskip

  (c) We have bound $\|P_{T(L)^\perp}(L^\star)\|$ in the proof for
  part (b). Here we use an alternative lower bound on $\sigma'$,
  namely
  \[
  \sigma' \geq \sigma_{\min}- \|\Delta_L\|
  \geq \frac{c_4 \la_n}{\xi(T)^2} - c_2\la_n
  \geq (c_4 -c_2)\la_n
  \geq \frac{3\alpha c_2^2 (2-\nu)}{16(3-\nu)} \la_n,
  \]
  where the second inequality follows from the gap assumption, the
  third inequality follows from $\xi(T)\leq 1$, and the last
  inequality follows from the definition of $c_4$.
  \medskip
	
  Using this alternative bound on $\sigma'$ gives the following
  alternative bound on $\|P_{T(L)^\perp}(L^\star)\|$,
  \[
  \|P_{T(L)^\perp}(L^\star)\|
  \leq \frac{\|\Delta_L\|^2}{\sigma'}
  \leq \frac{c_2^2\la_n^2}{\sigma'}
  \leq  \frac{16(3-\nu)}{3\alpha (2-\nu)} \la_n,
  \]
  where the first inequality follows from
  Lemma~\ref{lem:spectralcontrol}(b) similarly to the proof of
  part~(b) above, the second inequality follows from
  Proposition~\ref{prop:enforceparacons}, and the last inequality
  follows from the alternative lower bound on $\sigma'$.
\end{proof}

The claims in this section hold for all $(S, L)\in\M$ and hence also
apply to any minimizer $(S_\M, L_\M)$ of Problem~\ref{prob_variety}.
In the following, we work with an arbitrary solution $(S_\M,
L_\M)$. Later, we show that the solution is unique. In fact, we show
that it is the unique solution to the original Problem~\ref{prob_SL}.

\subsection{Step 2: Relaxation to a tangent space constrained problem}
\label{app_step2}

Problem~\ref{prob_variety} from the previous section is non-convex and
hence difficult to analyze further.  Hence, we consider a convex
relaxation of this problem where the non-convex rank constraint in the
definition of $\M$ is replaced by a low-rank tangent space constraint
at the solution of Problem~\ref{prob_variety}. Specifically, we now
focus on the following tangent space constrained problem
\begin{equation}
\begin{array}{lrlr}
  (S_\Y, L_\Y) = &\underset{S,\; L} {\argmin}
  &\ell(S+L) + \la_n \bPr{\gamma \|S\|_{1} + \|L\|_\ast} \\
  &\subjto& (S, L)\in \Y&
\end{array} \tag{\mbox{SL-$\Y$}}\label{prob_tspace}
\end{equation}
with the convex feasible set $\Y = \Omega\times T(L_\M)$.  Note that
we also dropped the technical constraints from the description of
$\M$. Later, we show that the solution automatically satisfies them.
\medskip

In this section, we proceed as follows. First, we show uniqueness of
the solution.  Then, we introduce some tools that will support the
work with optimality conditions.  Next, we discuss the choice of the
trade-off parameter $\la_n$. Afterwards, we show parametric
consistency of the solution $(S_\Y, L_\Y)$, particularly since at this
point we do not know that it is in $\M$.  Finally, we show that the
solutions coincide, that is, $(S_\Y, L_\Y)=(S_\M, L_\M)$.

\subsubsection{Uniqueness of the solution}

Here we show how transversality of the tangent spaces implies uniqueness
of the solution.

\begin{proposition}[uniqueness of the solution]
  \label{prop:tscuniquesol}
  Under the stability and $\gamma$-feasibility assumptions
  Problem~\ref{prob_tspace} is feasible and has a unique solution.
\end{proposition}
\begin{proof}
  Instead of showing uniqueness of the solution to
  Problem~\ref{prob_tspace} we consider the equivalent constrained
  form of the problem, that is,
  \[
  \min_{S,\, L}\: \ell(S+L) \quad \subjto \quad (S, L)\in\Y
  \:\textrm{ and }\: \|S\|_{1} \leq \tau_1,\: \|L\|_\ast \leq \tau_2,
  \]
  where $\tau_1, \tau_2 >0$ are constants that depend on $\la_n$ and
  $\gamma$. We show that this problem has a unique solution. First
  observe that the constraints of this problem describe a non-empty
  convex and \emph{compact} subset of $\textrm{Sym}(d)\times
  \textrm{Sym}(d)$. Hence, the existence of a solution follows from
  the convexity of the objective function, which is the composition of
  the negative log-likelihood function and the affine addition
  function.
  \medskip
	
  Uniqueness now follows from strict convexity of the objective, which
  we show as follows. For distinct $(S, L), (S', L')\in\Y$ at least
  one of $S-S'\in\Omega$ and $L-L'\in T(L_\M)$ must be non-zero.
  Therefore, for the compound matrices $\Theta = S+L$ and $\Theta'
  =S'+L'$ we have that
  \[
  \Theta - \Theta' = S + L - (S' + L') = (S-S') + (L-L') \neq 0
  \]
  because of the transversality of the tangent spaces, that is,
  $\Omega\,\cap\, T(L_\M) = \{0\}$. Transversality is a consequence of
  Proposition~\ref{prop_irreptransvcond} and the remark thereafter. We
  can apply Proposition~\ref{prop_irreptransvcond} since
  Corollary~\ref{cor:mconsconclus}(a) implies the necessary condition
  $\rho(T, T(L_\M))<\xi(T)/2$ for
  Proposition~\ref{prop_irreptransvcond} as $T(L_\M)$ is the tangent
  space at $L_\M$ and it holds $(S_\M, L_\M)\in\M$.  Now, by Taylor
  expansion with the mean-value form of the remainder it holds for
  some $t\in[0,1]$ that
  \begin{align*}
  \ell(\Theta') &= \ell(\Theta)+\nabla \ell(\Theta)^\top
  (\Theta-\Theta') + \frac{1}{2} (\Theta-\Theta')^\top \nabla^2
  \ell\bPr{t\Theta+(1-t)\Theta'}(\Theta-\Theta')\\
  &> \ell(\Theta)+\nabla \ell(\Theta)^\top (\Theta-\Theta'),
  \end{align*}
  where the inequality follows from $\Theta-\Theta'\neq0$ and the
  positive definiteness of the Hessian at any parameter matrix
  $\Theta$, that is, $M \nabla^2 \ell(\Theta)M > 0$ for all $0\neq
  M\in\textrm{Sym}(d)$.  This inequality establishes strict convexity
  of the objective as a function of $(S, L)\in \Y$.
  \medskip

  For completeness, we show how strict convexity implies
  uniqueness. For a contradiction suppose that $(S, L),\, (S',
  L')\in\Y$ are two different solutions. Then, strict convexity and
  equality of the objective values imply that
  \begin{align*}
    \ell(\frac{1}{2}(S+L) + \frac{1}{2}(S'+L')) < \frac{1}{2}\ell(S+L)
    + \frac{1}{2}\ell(S'+L') = \ell(S+L) = \ell(S'+L'),
  \end{align*}
  which contradicts that $(S, L)$ and $(S', L')$ are solutions. 
\end{proof}

We proceed with some auxiliary results.

\subsubsection{Supporting results for first-order optimality conditions}
\label{app:optimcond}

In our analysis we frequently use first-order optimality conditions.
In this section, we present some auxiliary results that
will turn out useful in the sequel.

\paragraph{Rewriting the gradient of the negative log-likelihood.}

The first-order optimality conditions particularly involve the
gradient of the negative log-likelihood.  Hence using $\Theta^\star=
S^\star + L^\star$, $\Delta_S = S_\Y - S^\star$ and $\Delta_L = L_\Y -
L^\star$ we rewrite the gradient $\nabla \ell(S_\Y + L_\Y)$ similar to
a Taylor-expansion as
\begin{align*}
  \nabla \ell(S_\Y+L_\Y) &= \nabla \ell(\Thstar + \Delta_S + \Delta_L) \\
	&= \nabla \ell(\Thstar) +
  H^\star(\Delta_S+\Delta_L) + R(\Delta_S+\Delta_L) \\
  &=\nabla \ell(\Thstar) + H^\star (\Delta_S + P_{T(L_\M)}\Delta_L) -
  H^\star P_{T(L_\M)^\perp}L^\star \\
	&\qquad + R(\Delta_S+ P_{T(L_\M)}\Delta_L -P_{T(L_\M)^\perp}L^\star),
\end{align*}
where we defined the remainder
\[
R(\Delta_S+\Delta_L) = \nabla \ell(S_\Y+L_\Y) - \nabla \ell(\Thstar) -
H^\star (\Delta_S+\Delta_L),
\]
and split $\Delta_L$ into its tangential and normal components for
reasons that will become evident later in the proof.  Now, the
remainder can be bounded using the following lemma.

\begin{lemma}
  \label{l_CgSLcons_boundedremainder}
  Let $\Delta_S\in \Omega$, $\gamma \in
  [\gamma_{\min},\gamma_{\max}]$, and assume that there exists a
  constant $r_0>0$ such that
  \begin{align*}
    \ganorm{(\Delta_S, \Delta_L)}\leq \max\left\{ 1,
    \frac{\nu\alpha}{2\beta(2-\nu)}\right\}^{-1} \frac{r_0}{2} = c_1.
  \end{align*}
  Then, there exists a constant $c_0>0$ such that the
  remainder can be bounded via
  \begin{align*}
    \ganorm{\D R(\Delta_S + \Delta_L)} \leq \frac{c_0}{\xi(T)}
    \ganorm{(\Delta_S, \Delta_L)}^2.
  \end{align*}
\end{lemma}
\begin{proof}
  First note that the gradient of the negative log-likelihood function
  and the gradient of the normalizer, that is, the log-partition
  function $a(\cdot)$, differ only by a constant since $\nabla
  \ell(\cdot) = \nabla a(\cdot)-\Phi^n$. Hence, the log-partition
  function can be similarly rewritten as
  \begin{align*}
    \nabla a(\Theta^\star+\Delta_S+\Delta_L) = \nabla a(\Thstar) + H^\star
    (\Delta_S+\Delta_L) + R(\Delta_S+\Delta_L)
  \end{align*}
  with the same Taylor-expansion remainder. Using a definite-integral
  representation we get for the remainder that
  \begin{align*}
    R(\Delta_S+\Delta_L)
    &= \nabla a(\Thstar + \Delta_S+\Delta_L) - \nabla a(\Thstar)
    - H^\star (\Delta_S+\Delta_L) \\
    &= \int_0^1 \bPe{\nabla^2 a\bPr{\Thstar + t(\Delta_S+\Delta_L)} -
    H^\star} (\Delta_S+\Delta_L) \,dt.
  \end{align*}
  For bounding the remainder, observe that the Hessian $\nabla^2 a$ is
  Lipschitz-continuous on any compact set since $a$ is twice
  continuously differentiable. Let $l(r_0)$ denote the Lipschitz
  constant for $\nabla^2 a$ on the compact ball $B(\Thstar) =
  \{\Theta: \|\Theta-\Thstar\|\leq r_0\}$, that is, for all $\Theta,
  \Theta'\in B(\Thstar)$ it holds
  \[
  \|\nabla^2 a(\Theta) - \nabla^2 a(\Theta')\| =
  \max_{M\,\in\,\textrm{Sym}(d):\,\|M\|=1} \|\bPe{\nabla^2 a(\Theta) -
    \nabla^2 a(\Theta')}M\|\leq l(r_0) \|\Theta-\Theta'\|.
  \]
  We establish that
  $\Thstar+\Delta_S+\Delta_L$ is contained in $B(\Thstar)$ by bounding
  \begin{align*}
    \|\Delta_S+\Delta_L\|
    &\leq \|\Delta_S\| + \|\Delta_L\| \\
    &\leq \gamma \mu(\Omega) \frac{\|\Delta_S\|_{\infty}}{\gamma }
    +\|\Delta_L\| \\
    &\leq 2\max\left\{ \gamma \mu(\Omega), 1\right\}
    \ganorm{(\Delta_S, \Delta_L)} \\
    &\leq 2\max\left\{ \frac{\nu\alpha}{2\beta(2-\nu)}, 1\right\}
    \ganorm{(\Delta_S, \Delta_L)} \leq r_0,
  \end{align*}
  where the second inequality holds by $\mu(\Omega)\geq1$, in the
  third inequality we upper-bounded the respective norms on $\Delta_S$
  and $\Delta_L$ by the $\ganorm{\cdot}$-norm, in the fourth
  inequality we used $\gamma\leq \gamma_{\max}$, and in the last
  inequality we used the assumed bound on $\ganorm{(\Delta_S,
    \Delta_L)}$.  Now, we use the Lipschitz constant $l(r_0)$ for the
  following spectral norm bound
  \begin{align*}
    \| R(\Delta_S+\Delta_L)\|
    &\leq \int_0^1 \bNorm{\bPe{\nabla^2 a\bPr{\Thstar + t(\Delta_S+\Delta_L)}
        - H^\star} (\Delta_S+\Delta_L)} dt\\
    &\leq \int_0^1 \bNorm{\nabla^2 a\bPr{\Thstar + t(\Delta_S+\Delta_L)}
      - H^\star} \bNorm {\Delta_S+\Delta_L} dt \\
    &\leq \int_0^1 l(r_0) t \bNorm{\Delta_S+\Delta_L}^2  dt \\
    &= \frac{l(r_0)}{2} \bNorm{\Delta_S+\Delta_L}^2.
  \end{align*}
 The claim now follows from
  \begin{align*}
    \ganorm{\D R(\Delta_S+\Delta_L)}
    &\leq \frac{\omega}{\xi(T)} \| R(\Delta_S+\Delta_L)\| \\
    &\leq \frac{l(r_0)\omega}{2\xi(T)} \|\Delta_S+\Delta_L\|^2 \\
    &\leq \frac{2l(r_0)\omega}{\xi(T)} 
    \max\left\{\frac{\nu\alpha}{2\beta(2-\nu)},1\right\}^2
    \ganorm{(\Delta_S, \Delta_L)}^2 \\
    &= \frac{c_0}{\xi(T)} \ganorm{(\Delta_S, \Delta_L)}^2,
  \end{align*}
  where we used Lemma~\ref{lem:normcompg2} in the first inequality,
  the bound on the remainder from above in the second inequality, and
  in the third inequality we used the bound on $\|\Delta_S+\Delta_L\|$
  which we also showed above. Finally, for the equality we defined
  \[
  c_0 = 2l(r_0)\omega
  \max\left\{\frac{\nu\alpha}{2\beta(2-\nu)},1\right\}^2.
  \]
\end{proof}

\paragraph{Subdifferential characterizations.}

Next, for our work with the first-order optimality conditions some
characterizations of norm subdifferentials will turn out useful.

\begin{lemma}
  \label{lem_sbdiff_dualnorm}
  Let $\|\cdot\|$ be a norm on $\IR^d$ and let $\|\cdot\|^\ast$ be its
  dual norm. Let $y\in\partial\|x\|$ for some $x\in\IR^d$.  Then, it
  holds for the dual norm that $\|y\|^\ast\leq 1$.
\end{lemma}
\begin{proof}
  Since $y\in\partial\|x\|$ it holds by the convexity of $\|\cdot\|$
  that
  \[
  \|z\| \geq \|x\| + y^\top (z-x)\quad \gdw \quad  y^\top x -
  \|x\| \geq y^\top z - \|z\|
  \]
  for all $z$, and thus also for the supremum over all $z$. This
  yields
  \[
  y^\top x - \|x\| \geq \sup_z \bPc{y^\top z - \|z\|}
  = \begin{cases}0, & \|y\|^\ast\leq 1 \\ \infty, &
    \text{else}\end{cases}
  \]
  since $\sup_z \bPc{y^\top z - \|z\|}$ is the Fenchel conjugate of
  the norm which is just the indicator function of the unit ball of
  the dual norm, see for example
  \cite[Proposition~1.4]{bach2012optimization}.  Now, the left hand
  side is always finite. Therefore, it must hold $\|y\|^\ast\leq1$.
\end{proof}

More specifically, we have the following characterizations of matrix
norm subdifferentials.

\begin{lemma}
  \label{lem_subdiff_tspace}
  \begin{itemize}
  \item[(a)] For $S\in \textrm{Sym}(d)$ with tangent space
    $\Omega(S)$ at $S$ to the variety of sparse matrices $\S(|\supp S|)$
    we have
    \[
    Z\in \partial \|S\|_{1} \quad\gdw\quad P_{\Omega(S)}(Z) = \sign(S)
    \,\textrm{ and }\,
    \|P_{\Omega(S)^\perp}(Z)\|_{\infty} \leq 1,
    \]
    where $\sign(S)\in\{-1,0,1\}^{d\times d}$ contains the signs of
    the entries of $S$ with the value $0$ for zero-entries.
  \item[(b)] For a rank-$r$ matrix $L \in \textrm{Sym}(d)$ let
    $L=UDU^T$ with $U\in\IR^{d\times r}$ and $D\in\IR^{r\times r}$ be
    its (restricted) singular value decomposition. Then, with the
    tangent space $T(L)$ at $L$ to the variety of low-rank matrices
    $\Lc(r)$ we have
    \[
    Z\in \partial \|L\|_\ast \quad \gdw\quad P_{T(L)}(Z) = U U^T
    \,\textrm{ and }\, \|P_{T(L)^\perp}(Z) \|\leq 1.
    \]
  \end{itemize}
\end{lemma}
\begin{proof}
  (a) The subdifferential of a sum of convex functions is the sum of
  the respective subdifferentials.  The $\ell_1$-norm is such a sum of
  convex functions, each of which maps onto the absolute value of a
  single particular entry. Elements in the subdifferential of such a
  function can only be non-zero in this particular entry, and the
  possible values of this particular entry are characterized by the
  subdifferential of the absolute value function $|\cdot|$.
\medskip

  Let $x\in\IR$. Then, if $x\neq0$ it holds apparently
  $\partial|x|=\sign(x)$, which corresponds to an entry being in the
  support of $S$. If $x=0$, we have $\partial|x|=[-1, 1]$, which
  corresponds to an entry being not in the support of $S$.  The proof
  for (a) is finished by noting that $P_{\Omega(S)}$ is precisely the
  projection on the entries that belong to the support, and that
  $P_{\Omega(S)^\perp}$ is the projection on the entries that are not
  contained in the support.

  \medskip

  (b) See~\cite[page~41]{watson1992characterization} .
\end{proof}

Finally, we state a result that characterizes the form of the
Lagrange multipliers for the tangent space constraints.
\begin{lemma}[Lagrange multipliers for linear subspace constraints]
  \label{leNagrangesubspace}
  Let $V \subset \IR^d$ be a vector space and let $f:V\to\IR$
  convex. Let $U\subseteq V$ be a linear subspace and $U^\perp$ its
  orthogonal complement.  Consider the problem
  \begin{equation*}
    \underset{x} {\min}\:\, f(x) \quad \subjto \quad x\in U.
  \end{equation*}
  Then, the Lagrangian for this problem is given by
  \[
  \Lc(x, \lambda) = f(x) + \lambda^\top x,
  \]
  where $\lambda \in U^\perp$.
\end{lemma}		
\begin{proof}
  Let $u_1, \ldots, u_k$ be a basis of $U^\perp$.  Then, the
  constraint $x\in U$ is equivalent to the set of constraints
  $u_i^\top x = 0$ for $i=1,\ldots, k$, that is, the component of $x$
  in any basis vector of $U^\perp$ must be zero.  Therefore, the
  Lagrangian can be represented as
  \[
  \Lc(x, \lambda_1, \ldots, \lambda_k) = f(x) + \sum_{i=1}^k \lambda_i
  u_i^\top x,
  \]
  where $\la_i\in\IR$ for $i=1,\ldots, k$.  Observe that $\lambda =
  \sum_i \lambda_i u_i \in U^\perp$ can be any element in $U^\perp$.
  This finishes the proof.
\end{proof}

\subsubsection{Choice of the trade-off parameter lambda}

The bounds on the individual errors can only be obtained if we
restrict the choice of the regularization parameter $\la_n$.  This is
very natural since bounding $\la_n$ ensures that the influence of the
regularization terms in the objective is neither too strong nor too
weak.  If $\la_n$ is too big then the shrinkage effects on the
solutions compared to the true parameters are too strong. On the other
hand, if $\la_n$ is too small it will be difficult to achieve
algebraic consistency.
\medskip

The function of $\la_n$ as a trade-off parameter also becomes evident
when working with the first-order optimality conditions of
Problem~\ref{prob_tspace}. Loosely speaking, we require that the
gradient \wrt the true parameters is small in terms of $\la_n$, that
is, having $\ganorm{\D \nabla \ell(\Thstar)}\leq C \la_n$ for some
small constant $C$.  Note that the stricter the bound on the true
gradient, the lower is the probability that this bound holds. This
also suggests that $\la_n$ should not be chosen too small. For our
analysis we set
\[
\la_n = \frac{6(2-\nu) }{\nu} \sqrt{\frac{c \kpa d\log d
    \|\Phi^\star\|}{n}} \frac{\omega}{\xi(T)}.
\]
Together with the assumption on the minimal number of samples this
choice of $\la_n$ implies the following upper bounds on $\la_n$ that
will particularly help to bound the individual errors, that is, to
show parametric consistency of the solution to
Problem~\ref{prob_tspace}.

\begin{lemma} \label{lem_upboundlan}
Assume that  
\[
n > c \kpa d\log d \|\Phi^\star\| \frac{\omega^2}{\xi(T)^2}
\bPe{\frac{\alpha\nu}{32(3-\nu)} \min\left\{\frac{c_1}{2},
    \frac{\alpha\nu\xi(T)}{128 c_0(3-\nu)}\right\}}^{-2}.
\]
Then, it follows that
\begin{subequations}
  \begin{align}
    \la_n &\leq \frac{3\alpha(2-\nu)}{32(3-\nu)} c_1, \label{labound_c1}\\
    \la_n &\leq \frac{3\alpha(2-\nu)}{16(3-\nu)}
    \frac{\alpha \nu \xi(T)}{128 c_0(3-\nu)}, \label{labound_128} \\
    \la_n &\leq \frac{3\alpha(2-\nu)}{16(3-\nu)}
    \frac{\alpha  \xi(T)}{32 c_0}. \label{labound_32}
  \end{align}
\end{subequations}	
\end{lemma}
\begin{proof}
  Plugging in the lower bound on $n$ yields that
  \[
  \la_n = \frac{6(2-\nu) }{\nu} \sqrt{\frac{c \kpa d\log d \|\Phi^\star\|}{n}}
  \frac{\omega}{\xi(T)} \leq\frac{3\alpha(2-\nu)}{16(3-\nu)}
  \min\left\{\frac{c_1}{2}, \frac{\alpha \nu \xi(T)}{128
    c_0(3-\nu)}\right\}.
  \]
  The first two bounds \eqref{labound_c1} and \eqref{labound_128} are
  direct consequences, and the last bound \eqref{labound_32} follows
  from $\nu\leq \frac{1}{2}$ since then $\frac{\nu}{4(3-\nu)}<1$.
\end{proof}

\subsubsection{Parametric consistency of the solution}

Now, we are prepared to show parametric consistency of the solution to
Problem~\ref{prob_tspace}.
      
\begin{proposition}
  \label{prop_tspaceerrbound}
  Assume that the stability, $\gamma$-feasibility, and gap assumptions
  hold. Further, assume that $\ganorm{\D \nabla
    \ell(\Thstar)}\leq\frac{\nu\la_n}{6(2-\nu)}$ and that the upper
  bounds \eqref{labound_c1} and \eqref{labound_32} on $\la_n$ are
  satisfied.  Then, the errors $\Delta_S=S_\Y-S^\star$ and $\Delta_L=
  L_\Y - L^\star$ are bounded as
  \[
  \ganorm{(\Delta_S, \Delta_L)}\leq \frac{32(3-\nu)}{3\alpha(2-\nu)}
  \la_n.
  \]
\end{proposition}
\begin{proof}
  For bounding the error $(\Delta_S,\Delta_L)$ we decompose it into
  its part in $\Y=\Omega\times T(L_\M)$ and its part in $\Y^\perp =
  \Omega^\perp \times T(L_\M)^\perp$, compare the proof of
  Proposition~\ref{prop:enforceparacons}. We have
  \begin{align*}
    (\Delta_S, \Delta_L) &= P_\Y(\Delta_S, \Delta_L) +
    P_{\Y^\perp} (\Delta_S, \Delta_L) \\
    &= \big(P_\Omega \Delta_S, P_{T(L_\M)} \Delta_L\big) + \big(P_{\Omega^\perp}
    \Delta_S, P_{T(L_\M)^\perp} \Delta_L \big)\\
    &= \big(S_\Y-S^\star, L_\Y - P_{T(L_\M)}L^\star\big) -
    \big(0, P_{T(L_\M)^\perp} L^\star\big).
  \end{align*}
  The $\ganorm{\cdot}$-norm of the second term $P_{\Y^\perp}
  (\Delta_S, \Delta_L) = -\big(0, P_{T(L_\M)^\perp} L^\star\big)$ is
  easy to bound since by Corollary~\ref{cor:mconsconclus}(c) we have
  that
  \[
  \ganorm{P_{\Y^\perp}(\Delta_S, \Delta_L)} = \ganorm{-(0,
    P_{T(L_\M)^\perp} L^\star)} = \bNorm{P_{T(L_\M)^\perp} L^\star
  } \leq \frac{16(3-\nu)}{3\alpha(2-\nu)}\la_n.
  \]

  \medskip

  If we can prove the same bound for the $\ganorm{\cdot}$-norm of the
  first term, that is, for $P_\Y(\Delta_S, \Delta_L)$, then the claim
  of the proposition follows immediately. Establishing this bound
  turns out to be more challenging. Here we are going to exploit the
  optimality conditions for Problem~\ref{prob_tspace}. The Lagrangian
  for Problem~\ref{prob_tspace} is given by
  \[
  \Lc(S, L, A_{\Omega^\perp}, A_{T(L_\M)^ {\perp}}) = \ell(S+L) +
  \la_n (\ga\|S\|_1 + \|L\|_\ast) + \langle{A_{\Omega^\perp},
    S}\rangle + \langle {A_{T(L_\M)^ {\perp}}, L} \rangle,
  \]
  where $A_{\Omega^\perp} \in \Omega^\perp$ and $A_{T(L_\M)^
    {\perp}}\in T(L_\M)^{\perp}$ are the Lagrange multipliers, see
  Lemma~\ref{leNagrangesubspace}. Now, the optimality conditions \wrt
  $S$ and $L$ for the optimal solution $(S_\Y,L_\Y)$ read as
  \begin{align*}
    0 = \nabla\ell(S_\Y+L_\Y)
    + Z_1 + A_{\Omega^\perp},\quad &Z_1 \in
    \la_n \ga\partial\|S_\Y\|_{1}, \qquad \text{and}\\
    0 = \nabla \ell(S_\Y+L_\Y)
    + Z_\ast  + A_{T(L_\M)^ {\perp}},\quad &Z_\ast \in
    \la_n\partial\|L_\Y\|_\ast.
  \end{align*}
  Using the projections onto $\Y$ and onto $\Y^\perp$, we can rewrite
  these conditions compactly as
  \[
  P_{\Y^\perp} \D\nabla \ell(S_\Y+L_\Y) = - P_{\Y^\perp} (Z_1, Z_\ast)
  - (A_{\Omega^\perp},A_{T(L_\M)^ {\perp}} )
  \]
  and
  \[
  P_\Y \D\nabla \ell(S_\Y+L_\Y) = -  P_\Y (Z_1, Z_\ast).
  \]
  Since the Lagrange multipliers are undetermined the first of these
  projected equations does not constitute a restriction on the optimal
  solution $(S_\Y,L_\Y)$. Instead, the optimal solution is fully
  characterized by the second projection, that is, the projection onto
  $\Y$. Hence, the solution to the second equation that we refer to as
  the projected optimality condition is also unique. The
  important observation about the projected optimality condition is
  that it is in fact a condition on the projected error
  $P_\Y(\Delta_S, \Delta_L) =\big(\Delta_S, P_{T(L_\M)}\Delta_L
  \big)$.  This follows immediately from
  \begin{align*}
    \nabla \ell(S_\Y+L_\Y) &=\nabla \ell(S_\Y - S^\star + S^\star + L_\Y
    - P_{T(L_\M)}L^\star + P_{T(L_\M)} L^\star) \\
    &=\nabla \ell(\Delta_S + S^\star + P_{T(L_\M)}\Delta_L + P_{T(L_\M)}L^\star)
  \end{align*}
  which does only depend on $\Delta_S$ and $P_{T(L_\M)}\Delta_L$ since
  $S^\star$ and $P_{T(L_\M)}L^\star$ are constants.
  \medskip
  
  We now show that $\ganorm{P_\Y(\Delta_S, \Delta_L)} \leq
  \frac{16(3-\nu)}{3\alpha(2-\nu)}\la_n$ by constructing a map whose
  only fixed-point is $P_\Y(\Delta_S, \Delta_L)$. Then, we show that a
  $\ganorm{\cdot}$-norm ball with radius
  $\frac{16(3-\nu)}{3\alpha(2-\nu)}\la_n$ is mapped onto itself by
  this map. This condition will allow the application of Brouwer's
  fixed-point theorem, which will
  guarantee the existence of a fixed-point \emph{within} this small
  ball. Since we already know that there is only one fixed-point, it
  must be the unique one, that is, $P_\Y(\Delta_S, \Delta_L)$. In this
  way, the desired bound on the projected error will follow.
  \medskip

  To construct the map we define the operator $\J:\Y\to\Y,\, (M,
  N)\mapsto P_\Y \D H^\star (M+N)$ and set $Z=-P_\Y (Z_1,
  Z_\ast)$. Then, we consider the continuous map
  \begin{align*}
    F(M, N) &= (M, N) - \J^{-1}\big(P_\Y \D
      \big[\nabla \ell(\Thstar) + H^\star (M+N) - H^\star P_{T(L_\M)^\perp}
        L^\star  \\
	&\qquad\qquad\qquad\qquad\qquad\qquad + R(M+N -P_{T(L_\M)^\perp}
        L^\star)\big]-Z\big) \\
      &= \J^{-1}\bPr{P_\Y \D \bPe{-\nabla \ell(\Thstar) +
          H^\star P_{T(L_\M)^\perp}L^\star - R(M+N - P_{T(L_\M)^\perp}L^\star)} +Z},
  \end{align*}
  where the inverse operator $\J^{-1}$ is well-defined since $\J$ is
  bijective on $\Y$. This is because from
  Corollary~\ref{cor:mconsconclus} we have $\rho(T, T(L_\M))\leq
  \frac{\xi(T)}{2}$ such that we can apply
  Proposition~\ref{prop_irreptransvcond}(a) which implies that $\J$ is
  injective.  The second equality then just uses that by definition
  $\J^{-1} P_\Y \D H^\star (M+N) = (M,N)$.  Observe that any fixed point
  $(M, N)$ of $F$ must satisfy
  \[
  Z = P_\Y \D [\nabla \ell(\Thstar) + H^\star (M+N) - H^\star
    P_{T(L_\M)^\perp}L^\star+ R(M+N -P_{T(L_\M)^\perp}L^\star)],
  \]
  that is, the optimality condition projected onto $\Y$ after the
  gradient $\nabla \ell(S_\Y + L_\Y)$ has been rewritten. As outlined
  above, we now show that $F$ maps a $\ganorm{\cdot}$-norm ball with
  radius $\frac{16(3-\nu)}{3\alpha(2-\nu)}\la_n$ onto itself, which
  then allows the application of Brouwer's fixed-point theorem.  If we
  assume that $\ganorm{(M, N)}\leq \frac{16(3-\nu)}{3\alpha(2-\nu)}
  \la_n$ this is shown by
  \begin{align*}
    &\ganorm{F(M, N)} \\
    &\quad \leq \frac{2}{\alpha} \ganorm{P_\Y \D \bPe{\nabla
        \ell(\Thstar) - H^\star P_{T(L_\M)^\perp}L^\star +
        R(M+N -P_{T(L_\M)^\perp}L^\star)} - Z} \\
    &\quad \leq \frac{2}{\alpha} \Big\{\ganorm{P_\Y \D \nabla
      \ell(\Thstar)} +\ganorm{P_\Y \D H^\star P_{T(L_\M)^\perp}L^\star} \\
	&\qquad + \ganorm{P_\Y \D  R(M+N - P_{T(L_\M)^\perp}L^\star)}
    + \ganorm{Z} \Big\} \\
    &\quad \leq \frac{4}{\alpha} \bPc{\ganorm{\D \nabla
        \ell(\Thstar) }+ \ganorm{\D H^\star P_{T(L_\M)^\perp}L^\star}+\la_n}
    +\frac{4}{\alpha}\ganorm{\D R(M+N -P_{T(L_\M)^\perp}L^\star)} \\
    &\quad \leq \frac{4}{\alpha}\bPr{\frac{2(3-\nu)}{3(2-\nu)}\la_n
      +\frac{2(3-\nu)}{3(2-\nu)}\la_n} \\
    &\quad = \frac{16(3-\nu)}{3\alpha(2-\nu)}\la_n,
  \end{align*}
  where the first inequality follows from
  Proposition~\ref{prop_irreptransvcond}(a), the second inequality is
  triangle inequality, and the third inequality is implied by
  Lemma~\ref{lem:projnorms} and
  \[
  \ganorm{Z} = \ganorm{P_\Y (Z_1, Z_\ast)} =
  \max \left\{\frac{\|P_\Omega (Z_1)\|_\infty}{\ga},
  \|P_{T(L_\M)} (Z_\ast)\|\right\}\leq 2\la_n
  \]  
  since by the subgradient characterizations, see
  Lemma~\ref{lem_subdiff_tspace} in Section~\ref{app:optimcond}, it
  holds that $ \|P_\Omega (Z_1)\|_\infty = \| \la_n\ga
  \sign(Z_1)\|_\infty \leq \la_n\gamma$, and by
  Lemma~\ref{lem_sbdiff_dualnorm} it holds that $\|Z_\ast\| \leq
  \la_n$, which yields $\|P_{T(L_\M)}(Z_\ast)\| \leq 2 \|Z_\ast\|\leq
  2\la_n$ in conjunction with Lemma~\ref{lem:projnorms}. Note that for
  bounding $\|P_{T(L_\M)} (Z_\ast)\|$ the subgradient characterization
  in Lemma~\ref{lem_subdiff_tspace} is not enough and we need
  Lemma~\ref{lem:projnorms}, because $Z_\ast$ is a subgradient at
  $L_\Y$ and the tangent space at $L_\Y$ can be different from
  $T(L_\M)$ although $L_\Y\in T(L_\M)$. The fourth and last inequality
  above is a consequence of on the one hand
  \begin{align}
  \ganorm{\D \nabla \ell(\Thstar) }+ \ganorm{\D H^\star
    P_{T(L_\M)^\perp}L^\star}+\la_n \leq \bPr{\frac{\nu}{3(2-\nu)} +
    1}\la_n = \frac{2(3-\nu)}{3(2-\nu)}\la_n \label{bound_tripleterm}
  \end{align}
  since $\ganorm{\D \nabla \ell(\Thstar) }\leq \frac{\nu
    \la_n}{6(2-\nu)}$ as assumed in this lemma and $\ganorm{\D H^\star
    P_{T(L_\M)^\perp}L^\star}\leq \frac{\nu \la_n}{6(2-\nu)}$ by
  Corollary~\ref{cor:mconsconclus}(b), and on the other hand from the
  fact that the remainder can be bounded by
  Lemma~\ref{l_CgSLcons_boundedremainder}, namely
  \begin{align}
    \ganorm{\D R(M+N -P_{T(L_\M)^\perp}L^\star)}
    &\leq  \frac{c_0}{\xi(T)} \ganorm{M, N-P_{T(L_\M)^\perp}L^\star}^2
    \notag\\
    &\leq\frac{c_0}{\xi(T)} \bPr{\ganorm{(M, N)} +
      \|P_{T(L_\M)^\perp}L^\star\|}^2 \notag\\
    &\leq  \frac{c_0}{\xi(T)} \bPr{\frac{32(3-\nu)}{3\alpha(2-\nu)}}^2\la_n^2
    \notag\\
    &\leq \frac{c_0}{\xi(T)} \bPr{\frac{32(3-\nu)}{3\alpha(2-\nu)}}^2\la_n
    \frac{3\alpha(2-\nu)}{16(3-\nu)}\frac{\alpha \xi(T)}{32 c_0} \notag\\
    &=\frac{2(3-\nu)}{3(2-\nu)}\la_n, \label{bound_taylorremainder}
  \end{align}
  where the second inequality is triangle inequality, the third
  inequality follows since $(M, N)$ is in the $\ganorm{\cdot}$-norm
  ball with radius $\frac{16(3-\nu)}{3\alpha(2-\nu)}\la_n$ and since
  from Corollary~\ref{cor:mconsconclus}(c) we have
  $\|P_{T(L_\M)^\perp}L^\star\| \leq \frac{16(3-\nu)}{3\alpha(2-\nu)}
  \la_n$, and in the last inequality we used the upper bound
  \eqref{labound_32} on $\la_n$.  The application of
  Lemma~\ref{l_CgSLcons_boundedremainder} was possible since the bound
  \eqref{labound_c1} on $\la_n$ implies that
  \[
  \ganorm{(M,N-P_{T(L_\M)^\perp}L^\star)} \leq \ganorm{(M, N)} +
  \|P_{T(L_\M)^\perp}L^\star\| \leq 
  \frac{32(3-\nu)}{3\alpha(2-\nu)}\la_n \leq c_1.
  \]
\end{proof}

Using the bound \eqref{labound_c1} on $\la_n$ we immediately get from
Proposition~\ref{prop_tspaceerrbound} that
\[
\ganorm{(\Delta_S, \Delta_L)}\leq \frac{32(3-\nu)}{3\alpha(2-\nu)}
\la_n \leq \frac{32(3-\nu)}{3\alpha(2-\nu)}
\frac{3\alpha(2-\nu)}{32(3-\nu)} c_1 = c_1.
\]

\subsubsection{Coinciding solutions}

Next, we show that the solution of the linearized problem indeed
coincides with the solution of the variety-constrained problem.  Since
in particular we need to show that the solution $(S_\Y, L_\Y)$ is in
$\M$, we start by showing that it satisfies the third constraint in
the description of $\M$.
\begin{proposition} \label{prop_satisfied3rdconstraint}
  Under the previously made assumptions, the solution $(S_\Y, L_\Y)$
  to Problem~\ref{prob_tspace} strictly satisfies the third constraint
  in the description of $\M$, that is,
\begin{align*}
  \ganorm{\D H^\star (\Delta_S+\Delta_L)} < 9 \la_n.
\end{align*}
\end{proposition}

\begin{proof}
  We compute that
  \begin{align*}
    &\ganorm{\D H^\star (\Delta_S+\Delta_L)} \\
    &\: = \ganorm{\D H^\star (\Delta_S+P_{T(L_\M)}\Delta_L -
      P_{T(L_\M)^\perp}L^\star)} \\
    &\: \leq \ganorm{P_\Y\D H^\star (\Delta_S+P_{T(L_\M)}\Delta_L)}
    + \ganorm{P_{\Y^\perp}\D H^\star (\Delta_S+P_{T(L_\M)}\Delta_L)}
    + \ganorm{\D H^\star P_{T(L_\M)^\perp}L^\star} \\
    &\: \leq \frac{40}{9}\la_n + \frac{40}{9}\la_n +
    \frac{\nu\la_n}{6(2-\nu)} \,\leq\, \frac{80}{9}\la_n +
    \frac{1}{18}\la_n \,<\, 9\la_n,
  \end{align*}
  where the equality used $\Delta_L = P_{T(L_\M)}\Delta_L -
  P_{T(L_\M)^\perp}L^\star$ as in the proof of
  Proposition~\ref{prop_tspaceerrbound}, the first inequality is the
  triangle inequality, and the third inequality uses
  $\frac{\nu}{6(2-\nu)}\leq\frac{1}{18}$ which is implied by $\nu\leq
  \frac{1}{2}$. The second inequality follows from
  Corollary~\ref{cor:mconsconclus}(b) that provides $\ganorm{\D
    H^\star P_{T(L_\M)^\perp}L^\star}\leq \frac{\nu\la_n}{6(2-\nu)}$,
  from Proposition~\ref{prop_irreptransvcond}(b) that gives
  \begin{align*}
    \ganorm{P_{\Y^\perp}\D H^\star (\Delta_S+P_{T(L_\M)}\Delta_L)}
    &\leq (1-\nu) \ganorm{P_\Y\D H^\star (\Delta_S+P_{T(L_\M)}\Delta_L)} \\
    &\leq \ganorm{P_\Y\D H^\star (\Delta_S+P_{T(L_\M)}\Delta_L)},
  \end{align*}
  and finally from rewriting the gradient that gives
  \begin{align*} 
    &\ganorm{P_\Y\D H^\star (\Delta_S + P_{T(L_\M)}\Delta_L)} \\
    &= \ganorm{P_\Y\D \bPe{\nabla \ell(S_\Y +L_\Y)
        -\nabla \ell(\Thstar) + H^\star P_{T(L_\M)^\perp}L^\star
        -R(\Delta_S+P_{T(L_\M)}\Delta_L-P_{T(L_\M)}L^\star) }} \\
    &= \ganorm{Z-P_\Y\D \bPe{\nabla \ell(\Thstar) -
        H^\star P_{T(L_\M)^\perp}L^\star + R(\Delta_S+P_{T(L_\M)}\Delta_L -
        P_{T(L_\M)}L^\star) }} \\
    &\leq \ganorm{Z} + \ganorm{P_\Y \D \nabla \ell(\Thstar)} +
    \ganorm{P_\Y\D H^\star P_{T(L_\M)^\perp}L^\star} \\
    &\quad + \ganorm{P_\Y \D R(\Delta_S+P_{T(L_\M)}\Delta_L-P_{T(L_\M)}L^\star)}\\
    &\leq 2\ganorm{\D R(\Delta_S+P_{T(L_\M)}\Delta_L-P_{T(L_\M)}L^\star}
    + 2\bPe{\ganorm{\D \nabla \ell(\Thstar)} +
      \ganorm{\D H^\star P_{T(L_\M)^\perp}L^\star} +\la_n} \\
    &\leq 2\bPe{\frac{2(3-\nu)\la_n}{3(2-\nu)} +
      \frac{2(3-\nu)\la_n}{3(2-\nu)}} \,=\,
    \frac{8(3-\nu)}{3(2-\nu)}\la_n \,\leq\, \frac{40}{9}\la_n,
  \end{align*}
  where the second equality follows from the projected optimality
  condition \[Z=-P_\Y(Z_1, Z_\ast) = P_\Y\D \nabla \ell(S_\Y +L_\Y),\]
  the first inequality is triangle inequality, the second inequality
  reorders and uses Lemma~\ref{lem:projnorms} as well as
  $\ganorm{Z}\leq 2\la_n$ (see the proof of
  Proposition~\ref{prop_tspaceerrbound}), the third inequality uses
  the bounds \eqref{bound_tripleterm} and
  \eqref{bound_taylorremainder} since the projected error
  $P_\Y(\Delta_S, \Delta_L)=(\Delta_S, P_{T(L_\M)}\Delta_L)$ is
  bounded by $\frac{16(3-\nu)}{3\alpha(2-\nu)}\la_n$ (see again the
  proof of Proposition~\ref{prop_tspaceerrbound}), and the last
  inequality uses $\nu\in(0,\frac{1}{2}]$ and thus
    $\frac{3-\nu}{2-\nu} \leq \frac{5}{3}$.
\end{proof}

For proving that the solutions coincide, we need two more lemmas that
are useful for relaxing variety constraints into tangent space
constraints.

\begin{lemma}[linearization lemma]
  \label{l_linearizationlemma}
  Let $E$ be some embedding space, let $f:E\to\IR$ be a convex continuous
  function, and let $\V\subset E$ be a variety. Assume that $\hat{x}$
  is a solution to the variety-constrained problem
  \[
  \begin{array}{lrlr}
    &\underset{x\in\V} {\min}&  f(x).
  \end{array}
  \]
  If $\hat{x}\in\V$ is a smooth point, then it is also a solution to
  the linearized problem
  \[
  \begin{array}{lrlr}
    &\underset{x\,\in\, \hat{x} + T_{\hat{x}}\V} {\min}&  f(x)
  \end{array}
  \]
  where $T_{\hat x}\V$ is the tangent space at $\hat x$ to the variety
  $\V$.
\end{lemma}
\begin{proof}
  The tangent space at $\hat{x}$ is given by the derivatives of
  differentiable curves passing through $\hat{x}$, that is,
  \[
  T_{\hat{x}}\V = \bPc{\gamma'(0):\;\: \gamma:(-1,1)\to\V,
    \gamma(0)=\hat{x}}.
  \]
  Now, let $0\neq\nu \in T_{\hat{x}}\V$ be a direction, and let
  $\gamma$ be any curve that yields this direction in the sense that
  $\gamma:(-1,1)\to\V, \gamma(0)=\hat x$, and $\gamma'(0)=\nu$.  Given
  that the derivative of the curve $\gamma$ at $0$ is defined as the
  limit $t\to0$ of $\frac{\gamma(t)\,-\,\gamma(0)}{t}$, we find that
  the points $\hat x + \frac{\gamma(t)\,-\,\gamma(0)}{t}$ converge to
  $\hat x + \nu$ as $t\to0$. Since $\nu\neq0$ we assume w.l.o.g.~that
  $\gamma(0)\neq\gamma(t)$ for all $t\neq0$. Observe that for $0<t<1$
  the points $\gamma(0)=\hat x$, $\gamma(t)$ and $\hat x +
  \frac{\gamma(t)\,-\,\gamma(0)}{t}= \left(1- \frac{1}{t}\right)
  \gamma(0) + \frac{\gamma(t)}{t}$ are colinear in this order, because
  $\frac{1}{t} >1$ and thus $1- \frac{1}{t}< 0$.
  \medskip
	
  Next, since $\hat x$ solves the original problem, the scalar
  function $f\circ \gamma: (-1, 1)\to\IR$ must have a minimum at $t=0$
  such that for any $t$ we have $f(\hat x)=f(\gamma(0)) \leq
  f(\gamma(t))$. Consequently, by the convexity of $f$ and colinearity
  it also holds that
  \[
  f(\hat x)=f(\gamma(0)) \leq f(\gamma(t)) \leq f\bPr{\hat x +
    \frac{\gamma(t)\,-\,\gamma(0)}{t}}.
  \]
  It follows from the continuity of $f$ and after taking the limit
  $t\to0$ that $f(\hat x) \leq f(\hat x + \nu)$.  Therefore, the claim
  of this lemma is a consequence of the arbitrariness of $\nu \in
  T_{\hat x} \V$.
\end{proof}

We briefly discuss what happens if another convex constraint is in
place.

\begin{lemma}[linearization with additional convex constraint]
  \label{l_linearizationlemma2}
  Let $E$ be some embedding space, let $f:E\to\IR$ be a convex
  continuous function, let $\V\subset E$ be a variety, and let
  $C\subset E$ be convex. Assume that $\hat{x}$ is a smooth point in
  $\V$ and that it solves the problem
  \begin{equation}
    \begin{array}{lrlr}
      &\underset{x\,\in \,C \cap \V} {\min}&  f(x).
    \end{array} \label{prob_VCconstrained}
  \end{equation}
	Suppose that $\hat x$ does not solve the linearized problem
\begin{equation*}
\begin{array}{lrlr}
 &\underset{x\,\in\, C\,\cap\,(\hat{x} + T_{\hat{x}}\V)} {\min}&  f(x).
\end{array}
\end{equation*}
Then, any solution to this problem must be on the boundary of $C$.
\end{lemma} 
\begin{proof}
  Let $0\neq\nu\in T_{\hat x}\V$ such that $\hat x + \nu$ is a solution
	to the linearized problem.
  Let $\gamma:(-1,1)\to\V$ be a smooth curve with $\gamma(0)=\hat x$,
  and $\gamma'(0)=\nu$.  Suppose that $f\circ\gamma:(-1,1)\to\IR$ has
  its minimum at zero. Then, by the proof of the previous lemma it
  would follow that $f(\hat x + \nu) \geq f(\hat x)$ which contradicts
  the assumption that $\hat x$ does not solve the linearized problem.
	Therefore, the value of $f$ must
  decrease locally around $\hat x$ along $\gamma$ and we can assume
  w.l.o.g.~that $f(\gamma(t))<f(\gamma(0))=f(\hat x)$ for all
  $t\in(0,1)$, that is, the curve enters the area where $f$ decreases
  for positive $t$. Now, since $\hat x$ solves
  Problem~\ref{prob_VCconstrained} it follows that $\gamma(t)$ for
  $t\in(0,1)$ cannot be feasible for this problem, that is, 
  $\gamma(t)\notin C$ for $t\in(0,1)$.
  \medskip

  Next, similarly to the proof of the previous lemma, for $0<t<1$ we
  consider the points $\gamma(0)=\hat x \in C$, $\gamma(t) \notin C$
  and $\hat x + \frac{\gamma(t)\,-\,\gamma(0)}{t}$, which are colinear
  in this order. Then, the convexity of $C$ implies that $\hat x +
  \frac{\gamma(t) \,-\, \gamma(0)}{t}\notin C$. Thus, since $\hat x +
  \frac{\gamma(t) \,-\, \gamma(0)}{t}\to \hat x + \nu$ as $t\to0$, we
  have shown that there are points arbitrarily close to $\hat x + \nu$
  that are not in $C$. Hence, the solution $\hat x + \nu$ cannot be in
	the interior of $C$.
\end{proof}

Now, we can finally show that the solutions coincide.

\begin{proposition}[coinciding solutions]
  \label{prop_coincidingsolutions}
  Under the assumptions made previously in
  Proposition~\ref{prop_tspaceerrbound}, the solutions of
  Problems~\ref{prob_tspace} and \ref{prob_variety} coincide, that is,
  $(S_\Y, L_\Y) = (S_\M, L_\M)$.
\end{proposition}
\begin{proof}
  Let us suppose for a contradiction that the solutions are not the
  same. Then, since
  $(S_\M, L_\M)\in \Omega\times T(L_\M)$
  is feasible for the uniquely solvable Problem~\ref{prob_tspace} the
  objective function value at $(S_\Y,L_\Y)$ must be smaller than the
  one at $(S_\M, L_\M)$.  We want to apply
  Lemma~\ref{l_linearizationlemma2} to the product $\V = \Omega \times
  \Lc(\rank L^\star)$ and the convex set
  \[
  C = \bPc{(S, L):\, \ganorm{\D H^\star (\Delta_{S}+\Delta_L)} \leq
    9\la_n}.
  \]
  We know that $(S_\M, L_\M)$ is a solution to
  \[
     \underset{(S,\, L)\, \in\, C\,\cap\, \V}
    {\min}\:\, \ell(S+L) + \la_n \bPr{\gamma \|S\|_{1} + \tr L},
  \]
  since the constraint $\|P_{T^\perp}(\Delta_L)\| \leq \frac{\xi(T)
    \la_n} {\omega \|H^ \star\|}$ in the description of $\M$ is
  non-binding by Corollary~\ref{cor:nonbinding4} and dropping this
  constraint from $\M$ yields the set $C\cap\V$.  By
  Proposition~\ref{prop:aconsistM} we know that the solution $\hat x =
  (S_\M, L_\M)$ is a smooth point. Note that in this case
  \[
  \hat x + T_{\hat x}\V = (S_\M, L_\M) + \Omega\times
  T(L_\M)=\Omega\times T(L_\M)
  \]
  such that the linearized problem is precisely
  Problem~\ref{prob_tspace} which is solved by $(S_\Y, L_\Y)$.
  \medskip
	
  Now, Lemma~\ref{l_linearizationlemma2} implies that $(S_\Y,L_\Y)$
  cannot be contained in the interior of $C$.  On the other hand, we
  have from Proposition~\ref{prop_satisfied3rdconstraint} that
  \[
  \ganorm{\D H^\star (\Delta_{S}+\Delta_L)} < 9\la_n,
  \]
  is strictly satisfied. Hence, $(S_\Y,L_\Y)$ is in fact contained in
  the interior of $C$.  This is contradiction and completes the proof.
\end{proof}

An easy consequence of the fact that the solutions coincide is that
the consistency properties from Proposition~\ref{prop:aconsistM} hold
for the solution to Problem~\ref{prob_tspace}.  In particular, we have
$\rank(L_\Y) = \rank (L^\star)$ and sign consistency, that is, $\sign
(S_\Y) = \sign (S^\star)$.  In addition, we have that $T(L_\Y) =
T(L_\M)$.

\medskip

Concerning parametric consistency we now have the bound
$\ganorm{(\Delta_S, \Delta_L)}\leq \frac{32(3-\nu)}{3\alpha(2-\nu)}
\la_n$ from Proposition~\ref{prop_tspaceerrbound} and since we showed
that the solution is also in $\M$ we also have the bound
$\ganorm{(\Delta_S, \Delta_L)}\leq c_2\la_n = \bPr{\frac{40}{\alpha} +
  \frac{1}{\|H^\star\|}}\la_n$ from
Proposition~\ref{prop:enforceparacons}.  An easy calculation
demonstrates that $\frac{32(3-\nu)}{3\alpha(2-\nu)} \leq
\frac{40}{\alpha}\leq \frac{40}{\alpha} + \frac{1}{\|H^\star\|}$ and
thus that the first bound is always better.


\subsection{Step 3: Removing tangent space constraints}
\label{app_step3}

In this section, we show how using upper bounds on $\la_n$ also allows
us to conclude that the solution $(S_\Y, L_\Y)$ to
Problem~\ref{prob_tspace} satisfies the optimality conditions of the
problem
\begin{equation}
\begin{array}{lrlr}
 &\underset{S,\, L} {\min}& \ell(S+L) + \la_n \bPr{\gamma \|S\|_{1} +
    \|L\|_\ast},
\end{array}  \tag{\mbox{SL-$\emptyset$}} \label{prob_PDdropped}
\end{equation}
with the tangent space constraints removed.  Note that the only
difference to the original Problem~\ref{prob_SL} is the absence of the
positive-semidefiniteness constraint. However, since the solution
$(S_\Y, L_\Y)=(S_\M, L_\M)$ satisfies $L_\Y\succeq0$ by
Proposition~\ref{prop:aconsistM}(a) we know that if it is a solution
to Problem~\ref{prob_PDdropped} it is also a solution to the original
Problem~\ref{prob_SL}.

\subsubsection{Primal-dual witness condition}

In this section, we show that given a \emph{strictly dual feasible}
solution to Problem~\ref{prob_PDdropped} that is contained in the
linearized correct model space $\Y$ there cannot be other solutions to
Problem~\ref{prob_PDdropped} that are not contained in $\Y$. Hence,
the strictly dual feasible solution acts as a \emph{witness}.

\begin{proposition}[primal-dual witness]
  \label{prop_PDWcondition}
  Let $(S_\Y, L_\Y)$ be a solution in $\Y = \Omega \times T(L_\M)$ to
  Problem~\ref{prob_PDdropped} with corresponding subgradients $Z_1\in
  \la_n\ga\partial\|S_\Y\|_1$ and $Z_\ast\in \la_n
  \partial\|L_\Y\|_\ast$ such that it holds
  \begin{align*}
    \nabla\ell(S_\Y+L_\Y) + Z_1 = 0,\;\; \tand \;\;\nabla
    \ell(S_\Y+L_\Y) + Z_\ast = 0.
  \end{align*}
  Suppose that the subgradients satisfy the strict dual feasibility
  condition
  \[
  \ganorm{P_{\Y^\perp} (Z_1, Z_\ast)} =
  \max\bPr{\frac{\|P_{\Omega^\perp}Z_1\|_\infty}{\gamma},
    \|P_{T(L_\M)^\perp}Z_\ast\|} < \la_n.
  \]
  Then all solutions to Problem~\ref{prob_PDdropped} must be in $\Y$.
\end{proposition}
\begin{proof}
  Let $(S_\Y + M, L_\Y + N)$ be another solution to
  Problem~\ref{prob_PDdropped}. Our goal is to show that $M\in\Omega$
  and $N\in T(L_\M)$. First, it follows from the equality of the
  optimal objective function values that
  \begin{align*}
    0&= \ell(S_\Y +M + L_\Y + N) + \la_n\bPr{\ga \|S_\Y + M\|_1 +
      \|L_\Y + N\|_\ast} \\
    &\qquad - \ell(S_\Y + L_\Y) - \la_n\bPr{\ga \|S_\Y\|_1 + \|L_\Y\|_\ast} \\
    &\geq  \scalp{\nabla \ell(S_\Y + L_\Y) + Q_1, M} +
    \scalp{\nabla \ell(S_\Y + L_\Y) + Q_\ast, N} \\
    &= \scalp{Q_1- Z_1, M} + \scalp{Q_\ast - Z_\ast, N} \\
    &= \scalp{P_{\Omega^\perp} (Q_1- Z_1), M} +
    \scalp{P_{T(L_\M)^\perp}(Q_\ast - Z_\ast), N}\\
    &= \scalp{P_{\Omega^\perp} (Q_1- Z_1), P_{\Omega^\perp} M} +
    \scalp{P_{T(L_\M)^\perp}(Q_\ast - Z_\ast), P_{T(L_\M)^\perp} N},
  \end{align*}
  where in the inequality we bounded the objective function value at
  $(S_\Y + M, L_\Y + N)$ using a subgradient of the convex objective
  function at $(S_\Y, L_\Y)$. The subgradient of the objective
  function is composed of the gradient $\nabla \ell(S_\Y+L_\Y)$ of the
  negative log-likelihood (note that the derivatives \wrt to $S$ and
  $L$ coincide), and of some subgradients $Q_1\in
  \la_n\ga\partial\|S_\Y\|_1$ and $Q_\ast\in \la_n
  \partial\|L_\Y\|_\ast$ that we can choose.  Later we will make
  explicit choices.  In the further steps of the calculation above we
  used the optimality condition for the solution $(S_\Y, L_\Y)$ in the
  second equality. In the third equality we used that $Q_1, Z_1\in
  \la_n\ga\partial\|S_\Y\|_1$ and $Q_\ast, Z_\ast\in \la_n
  \partial\|L_\Y\|_\ast$, and that consequently their components in
  the respective tangent spaces $\Omega$ and $T$ must coincide by the
  subgradient characterizations in
  Lemma~\ref{lem_subdiff_tspace}. Therefore, these components cancel
  each other out and only the projections onto the orthogonal complements
  of the tangent spaces remain.
  \medskip

  We now choose suitable components of $Q_1$ and $Q_\ast$ in
  $\Omega^\perp$ and $T(L_\M)^\perp$, respectively. By the subgradient
  characterization in Lemma~\ref{lem_subdiff_tspace} our only
  restriction is that it must hold $\|P_{\Omega^\perp}
  Q_1\|_\infty\leq \la_n\ga$ and $\|P_{T(L_\M)^\perp} Q_\ast\|\leq
  \la_n$. We choose $P_{\Omega^\perp} Q_1 = \la_n\ga \sign
  \bPr{P_{\Omega^\perp}M}$ and $P_{T(L_\M)^\perp} Q_\ast = \la_n O
  \sign (\Sigma) O^\top$, where $P_{T(L_\M)^\perp} N = O \Sigma O^\top$
  is an eigenvalue decomposition of $P_{T(L_\M)^\perp} N$ with
  orthogonal $O\in \IR^{d\times d}$ and diagonal $\Sigma
  \in\IR^{d\times d}$. It can be readily checked that indeed
  $\|P_{\Omega^\perp} Q_1\|_\infty\leq \la_n\ga$ and
  $\|P_{T(L_\M)^\perp} Q_\ast\|\leq \la_n$.  Now, we continue the
  calculation from above with the specific subgradients
  \begin{align*}
  &\scalp{P_{\Omega^\perp} (Q_1- Z_1), P_{\Omega^\perp} M} +
    \scalp{P_{T(L_\M)^\perp}(Q_\ast - Z_\ast), P_{T(L_\M)^\perp} N} \\
  &= \scalp{P_{\Omega^\perp} Q_1, P_{\Omega^\perp} M} -
    \scalp{P_{\Omega^\perp} Z_1, P_{\Omega^\perp} M} \\
  &\quad +
    \scalp{P_{T(L_\M)^\perp}Q_\ast, P_{T(L_\M)^\perp} N}
		-\scalp{P_{T(L_\M)^\perp} Z_\ast, P_{T(L_\M)^\perp} N} \\
  &= \la_n\ga\|P_{\Omega^\perp} M\|_1 - \scalp{P_{\Omega^\perp} Z_1,
      P_{\Omega^\perp} M} + \la_n\|P_{T(L_\M)^\perp} N\|_\ast -
    \scalp{P_{T(L_\M)^\perp} Z_\ast, P_{T(L_\M)^\perp} N} \\
  &\geq \la_n\ga\|P_{\Omega^\perp} M\|_1 - \|P_{\Omega^\perp}
    Z_1\|_\infty \|P_{\Omega^\perp} M\|_1 + \la_n\|P_{T(L_\M)^\perp} N\|_\ast
    - \|P_{T(L_\M)^\perp} Z_\ast\| \|P_{T(L_\M)^\perp} N\|_\ast \\
  &= \bPr{\la_n\ga - \|P_{\Omega^\perp} Z_1\|_\infty}\|P_{\Omega^\perp}
    M\|_1 + \bPr{\la_n - \|P_{T(L_\M)^\perp}
      Z_\ast\|}\|P_{T(L_\M)^\perp} N\|_\ast,
  \end{align*}
  where the second equality follows from
  \[
    \scalp{P_{\Omega^\perp} Q_1, P_{\Omega^\perp} M} = \scalp{\la_n\ga
      \sign \bPr{P_{\Omega^\perp}M}, P_{\Omega^\perp} M} = \la_n
    \ga\|P_{\Omega^\perp} M\|_1
  \]
  and 
  \begin{align*}
    \scalp{P_{T(L_\M)^\perp}Q_\ast, P_{T^\perp} N}
    &= \scalp{\la_n O \sign (\Sigma) O^\top, O \Sigma O^\top} \\
    &= \la_n \trace\bPr{\bPr{O \sign (\Sigma) O^\top}^\top O \Sigma O^\top} \\
    &= \la_n \trace\bPr{O \sign (\Sigma) O^\top O \Sigma O^\top} \\
    &= \la_n \trace\bPr{O \sign (\Sigma) \Sigma O^\top}
    = \la_n \trace\bPr{O |\Sigma| O^\top} \\
    &= \la_n \trace\bPr{ |\Sigma| O^\top O} = \la_n \tr (|\Sigma|)
    = \la_n \|P_{T(L_\M)^\perp} N\|_\ast,
  \end{align*}
  and the inequality follows from the (generalized) H\"older's
  inequality for the respective dual norm pairs ($l_\infty$- and
  $l_1$-norm, and nuclear and spectral norm).  In summary, we now have
  \[
  0\geq \bPr{\la_n\ga - \|P_{\Omega^\perp}  Z_1\|_\infty}\|P_{\Omega^\perp} M\|_1
  + \bPr{\la_n - \|P_{T(L_\M)^\perp}  Z_\ast\|}\|P_{T(L_\M)^\perp} N\|_\ast.
  \]
  From the assumption $\ganorm{P_{\Y^\perp} (Z_1, Z_\ast)} =
  \max\bPr{\frac{\|P_{\Omega^\perp}Z_1\|_\infty}{\gamma},
    \|P_{T(L_\M)^\perp}Z_\ast\|} < \la_n$ it follows that
  $\|P_{\Omega^\perp} Z_1\|_\infty<\la_n\ga$ and $\|P_{T(L_\M)^\perp}
  Z_\ast\|<\la_n$. Therefore, the equality above can only hold if both
  $\|P_{\Omega^\perp} M\|_1 = 0 $ and $\|P_{T(L_\M)^\perp}
  N\|_\ast=0$, that is, if $M\in\Omega$ and $N\in T(L_\M)$. This implies
  that $S_\Y + M \in \Omega$ and $L_\Y + N\in T(L_\M)$. In other
  words, the other solution $(S_\Y + M, L_\Y + N)$ is also contained
  in $\Y=\Omega\times T(L_\M)$. This finishes the proof.
\end{proof}

\subsubsection{Coinciding solutions}

Finally, we show that the solution $(S_\Y, L_\Y)$ to the tangent space
constrained problem is also the unique solution to the original
Problem~\ref{prob_SL}. 

\begin{proposition}[coinciding solutions]
  \label{prop_strictdualfeas}
  Assume that the upper bounds \eqref{labound_c1} - \eqref{labound_32}
  are satisfied by $\la_n$ and assume that $\ganorm{\D \nabla
    \ell(\Thstar)}\leq\frac{\nu\la_n}{6(2-\nu)}$.  Then, under the
  stability, $\gamma$-feasibility, and gap assumptions, the solution
  $(S_\Y, L_\Y)$ to the tangent space constrained
  Problem~\ref{prob_tspace} also uniquely solves
  Problem~\ref{prob_SL}.
\end{proposition}
\begin{proof}
  First note that it suffices to show that $(S_\Y, L_\Y)$ uniquely
  solves Problem~\ref{prob_PDdropped}. This is because $(S_\Y, L_\Y)$
  is in $\M$ by Proposition~\ref{prop_coincidingsolutions} and
  therefore it holds $L_\Y\succeq0$ by
  Proposition~\ref{prop:aconsistM}(a). Hence, on the one hand we need
	to prove that $(S_\Y, L_\Y)$ solves Problem~\ref{prob_PDdropped},
	and on the other hand we must show that it is the unique solution.
  \medskip
  
  We show that $(S_\Y, L_\Y)$ solves Problem~\ref{prob_PDdropped} by
  verifying the first-order optimality conditions, that is, by showing
  that
  \[
  \nabla\ell(S_\Y+L_\Y) \in - \la_n \gamma\partial\|S_\Y\|_{1} \quad
  \tand\quad \nabla \ell(S_\Y+L_\Y) \in - \la_n \partial\|L_\Y\|_\ast.
  \]
  Hence, we need to check that $\nabla\ell(S_\Y+L_\Y)$ satisifies the
  norm-subdifferential characterizations in
  Lemma~\ref{lem_subdiff_tspace} that can be written as
  \[
  P_\Y \D\nabla \ell(S_\Y+L_\Y) = - \la_n (\gamma \sign(S_\Y),
  UU^\top) \quad\tand\quad \ganorm{P_{\Y^\perp} \D\nabla
    \ell(S_\Y+L_\Y)} \leq\la_n,
  \]
	where $L_\Y=UDU^\top$ is an eigenvalue decomposition of $L_\Y$.
  The first condition is almost the same as the optimality condition
  of Problem~\ref{prob_tspace}, where we had two additional Lagrange
  multipliers in $\Omega^\perp$ and $T(L_\M)^\perp$, respectively, due
  to the tangent space constraints. These Lagrange multipliers vanish
  here since we consider the projection onto the tangent space
  $\Y=\Omega\times T(L_\M)$.  Therefore, the
  optimality/subdifferential conditions of both problems projected onto
  the components in $\Y$ coincide, and we know that they are satisfied
  by $(S_\Y, L_\Y)$ which is the unique solution in $\Y$ to this
  projected optimality condition. Hence, we have indeed $P_\Y
  \D\nabla \ell(S_\Y+L_\Y) = - \la_n (\gamma \sign(S_\Y), UU^\top)$.
  For future reference we set $Z= P_\Y \D\nabla \ell(S_\Y+L_\Y)$ again
  and note that
  \[
  \ganorm{Z} = \ganorm{P_\Y \D\nabla \ell(S_\Y+L_\Y)} = \|{- \la_n
    (\gamma \sign(S_\Y), UU^\top)}\|_\gamma \leq\la_n.
  \]
  Actually, it holds equality, if not both $S_\Y$ and $L_\Y$ are zero.
  \medskip

  For showing that $(S_\Y, L_\Y)$ solves Problem~\ref{prob_PDdropped}
  it remains to establish the second condition $\ganorm{P_{\Y^\perp}
    \D\nabla \ell(S_\Y+L_\Y)} \leq\la_n$. Here, we show the stronger
  sharp inequality, that is, strict dual feasibility. In conjunction with
  Proposition~\ref{prop_PDWcondition} this immediately implies that
  $(S_\Y, L_\Y)$ is the only solution to Problem~\ref{prob_PDdropped},
  because then $(S_\Y, L_\Y)$ can be used as a witness in the sense of this
  proposition which implies that all solutions to
  Problem~\ref{prob_PDdropped} must be in $\Y$. Thus $(S_\Y, L_\Y)$
  must be the unique solution, since we already know that this is the
  only solution in $\Y$ to the projected optimality condition (that
  concerns the components in $\Y$).
  \medskip

  The rest of the proof is dedicated to showing strict dual
  feasibility. We do so by leveraging on the Taylor expansion as in
  the proof of Proposition~\ref{prop_tspaceerrbound} again, namely
  \begin{align*}
    &\ganorm{P_{\Y^\perp} \D\nabla \ell(S_\Y+L_\Y)} \\
    &\quad = \ganorm{P_{\Y^\perp} \D[\nabla \ell(\Thstar) +
        H^\star(\Delta_S+P_{T(L_\M)}\Delta_L) - H^\star P_{T(L_\M)^\perp}L^\star+
        R(\Delta_S+\Delta_L)] } \\
    &\quad\leq \ganorm{P_{\Y^\perp} \D H^\star(\Delta_S+P_{T(L_\M)}\Delta_L)}
    + \ganorm{P_{\Y^\perp} \D [\nabla \ell(\Thstar) -H^\star P_{T(L_\M)^\perp}
        L^\star+   R(\Delta_S+\Delta_L)]}\\
    &\quad<\la_n,
  \end{align*}
  where the first inequality is triangle inequality, and the second
  one needs some more elaboration. To show it we start by applying
  Proposition~\ref{prop_irreptransvcond}(b) such that
  \begin{align*}
    &\ganorm{P_{\Y^\perp} \D H^\star (\Delta_S+P_{T(L_\M)}\Delta_L)} \\
    &\qquad\leq (1-\nu)  \ganorm{P_{\Y} \D H^\star
      (\Delta_S+P_{T(L_\M)}\Delta_L)} \\
    &\qquad= (1-\nu)\ganorm{P_\Y \D\nabla \ell(S_\Y+L_\Y) - P_{\Y}
      \D[\nabla \ell(\Thstar) -H^\star P_{T(L_\M)^\perp}L^\star +
        R(\Delta_S+\Delta_L)]} \\
    &\qquad= (1-\nu)\ganorm{Z - P_{\Y} \D[\nabla \ell(\Thstar) - H^\star
        P_{T(L_\M)^\perp}L^\star+   R(\Delta_S+\Delta_L)]} \\
    &\qquad\leq (1-\nu)\bPc{ \ganorm{Z} + \ganorm{P_\Y\D\bPe{\nabla
          \ell(\Thstar) -H^\star P_{T(L_\M)^\perp}L^\star +
          R(\Delta_S+\Delta_L)}} } \\
    &\qquad\leq (1-\nu)\bPc{\la_n +2 \ganorm{\D\bPe{\nabla \ell(\Thstar)
          - H^\star P_{T(L_\M)^\perp}L^\star +  R(\Delta_S+\Delta_L)}} } \\
    &\qquad\leq (1-\nu)\bPc{\la_n +\frac{\nu\la_n}{2-\nu}} =
    \frac{2\la_n(1-\nu)}{2-\nu}= \la_n - \frac{\nu\la_n}{2-\nu} \\
    &\qquad<\la_n - \frac{\nu\la_n}{2(2-\nu)} \\
    &\qquad\leq \la_n -  \ganorm{\D[\nabla \ell(\Thstar) -H^\star
        P_{T(L_\M)^\perp}L^\star+   R(\Delta_S+\Delta_L)]} \\
    &\qquad\leq \la_n -  \ganorm{P_{\Y^\perp}\D[\nabla \ell(\Thstar)
        - H^\star P_{T(L_\M)^\perp}L^\star+   R(\Delta_S+\Delta_L)]},
  \end{align*}
  where the equalities use the Taylor expansion and the definition of
  $Z$, the second inequality is triangle inequality, the third
  and last inequality use Lemma~\ref{lem:projnorms} and
  $\ganorm{Z}\leq\la_n$, and the fourth and second-to-last inequality
  follow from
  \begin{align*}
    &\ganorm{\D\bPe{\nabla \ell(\Thstar) -H^\star P_{T(L_\M)^\perp}L^\star +
        R(\Delta_S+\Delta_L)}} \\
    &\qquad\leq \ganorm{\D \nabla \ell(\Thstar)} +
    \ganorm{\D H^\star P_{T(L_\M)^\perp}L^\star} + \ganorm{\D R(\Delta_S+\Delta_L)} \\
    &\qquad\leq 3 \frac{\nu\la_n}{6(2-\nu)} = \frac{\nu\la_n}{2(2-\nu)},
  \end{align*}
  which follows from triangle inequality and the fact that $\ganorm{\D
    \nabla \ell(\Thstar)}\leq \frac{\nu\la_n}{6(2-\nu)}$ by
  assumption, $\ganorm{\D H^\star P_{T(L_\M)^\perp}L^\star}\leq
  \frac{\nu\la_n}{6(2-\nu)}$ by Corollary~\ref{cor:mconsconclus}(b),
  and that the remainder too can be bounded
  \begin{align*}
    \ganorm{\D R(\Delta_S+ \Delta_L)} &\leq  \frac{c_0}{\xi(T)}
    \ganorm{(\Delta_S, \Delta_L)}^2\\
    &\leq \frac{c_0}{\xi(T)}  \bPr{\frac{32(3-\nu)}{3\alpha(2-\nu)}}^2
    \lambda_n^2  \\ 
    &\leq \frac{c_0}{\xi(T)}  \bPr{\frac{32(3-\nu)}{3\alpha(2-\nu)}}^2
    \lambda_n  \frac{3\alpha(2-\nu)}{16(3-\nu)}
    \frac{\alpha\nu\xi(T)}{128 c_0(3-\nu)}\\
    &= \frac{\nu\la_n}{6(2-\nu)},
  \end{align*}
  where we used Lemma~\ref{l_CgSLcons_boundedremainder} in the first
  inequality, which is possible since from
  Proposition~\ref{prop_tspaceerrbound} it follows that
  $\ganorm{(\Delta_S, \Delta_L)}\leq
  \frac{32(3-\nu)}{3\alpha(2-\nu)}\la_n\leq c_1$.  Note that this also
  explains the second inequality. Finally, the last inequality is a
  consequence of the upper bound \eqref{labound_128} on $\la_n$.
\end{proof}

\subsection{Step 4: Probabilistic analysis and completion of the proof}

\subsubsection{Probabilistic analysis}
\label{sec:probanalysis}
In this section, we need to bound the gradient of the negative
log-likelihood.  We begin by citing a result from random matrix
theory.

\begin{theorem}[Corollary 5.52, \cite{vershynin2010introduction}]
  \label{vershynin_generalcovmats}
  Let $X\in\IR^d$ be a random vector that satisfies $\|X\|^2\leq m$
  almost surely.  Let $x^{(1)}, \ldots, x^{(n)}$ be $n$ independent
  observations of $X$.  By $\Sigma = \IE \bPe{X X^\top}$ we denote the
  expected second-moment matrix of $X$ and by $\Sigma^n =
  \frac{1}{n}\sum_{k=1}^n x^{(k)}\bPe{x^{(k)}}^\top$ the empirical
  second-moment matrix.  Assume that $\Sigma$ is invertible and let $0 <
  \varepsilon <1$. If the number of samples satisfies $n\geq c
  (t/\ep)^2 \|\Sigma\|^{-1}m \log d$ for some $t\geq1$, then we have
  \begin{align*}
    \IP\bPr{\bNorm{\Sigma^n - \Sigma} > \ep \|\Sigma\|} \leq \exp\bPr{-t^2 +
      \log d} = d^{-t^2}.
  \end{align*}
	Here, $c$ is an absolute constant.
\end{theorem}

First note that by substituting $\delta = \ep \|\Sigma\|$ and $\kpa =
t^2$ we get the following simple corollary.

\begin{corollary} \label{coro_generalcovmats_kpa}
  Let $X\in\IR^d$ be a random vector that satisfies $\|X\|^2\leq m$
  almost surely and let $\Sigma$ and $\Sigma^n$ be defined as before.  Let
  $0 <\delta < \|\Sigma\|$ and $\kpa\geq1$.  If the number of
  samples satisfies $n\geq c \kpa \delta^{-2} \|\Sigma\| m \log d$, then we
  have
  \begin{align*}
    \IP\bPr{\bNorm{\Sigma^n - \Sigma} > \delta } &\leq
    d^{-\kpa}.  \label{vershynin_generalcovmats_v2}
  \end{align*}
\end{corollary}

Another corollary is the following.
\begin{corollary} \label{coro_probdeltan}
  Let $X\in\IR^d$ be a random vector that satisfies $\|X\|^2\leq m$
  almost surely and let $\Sigma$ and $\Sigma^n$ be defined as before.  Let
  $\kpa\geq1$ and let $\delta_n = \sqrt{\frac{c \kpa \|\Sigma\| m\log
      d}{n}}$.  If $n> c \kpa \|\Sigma\|^{-1} m\log d$, then we have
  \begin{align*}
    \IP\bPr{\bNorm{\Sigma^n - \Sigma} > \delta_n } &\leq d^{-\kpa}. 
  \end{align*}
\end{corollary}

\begin{proof}
  We apply Corollary~\ref{coro_generalcovmats_kpa} with
  $\delta=\delta_n$.  Since
  \[
  c \kpa \delta_n^{-2}\|\Sigma\| m \log d = n \leq n,
  \]
  the lower bound on the number of sample points from
  Corollary~\ref{coro_generalcovmats_kpa} is satisfied.  Moreover, the
  condition $n> c \kpa \|\Sigma\|^{-1} m\log d$ implies
  \[
  \delta_n^2 = \frac{c \kpa \|\Sigma\| m\log d}{n} < \frac{c \kpa
    \|\Sigma\| m\log d}{c \kpa \|\Sigma\|^{-1} m\log d} = \|\Sigma\|^2.
  \]
  This verifies the last necessary condition $\delta_n<\|\Sigma\|$ and
  thereby concludes the proof.
\end{proof}

Note that corollary~\ref{coro_probdeltan} needs a higher number of
samples if the spectral norm $\|\Sigma\|$ is small, however $\delta_n$
is also smaller if $\|\Sigma\|$ is small.
\medskip

We now use Corollary~\ref{coro_probdeltan} to bound the gradient of
the likelihood for our Ising random vectors.  Recall that $\Phi^\star
= \IE_{\Thstar}[\Phi] \in\IR^{d\times d}$ is the expected value of the
sufficient statistics under the multivariate Ising distribution with
parameter matrix $\Thstar=S^\star+L^\star$, and that $\Phi^n =
\frac{1}{n}\sum_{k=1}^n x^{(k)} \bPe{x^{(k)}}^\top$ is the corresponding
empirical version which is based on $n$ observations $x^{(1)}, \ldots,
x^{(n)}$.

\begin{corollary}[bound on the Ising log-likelihood gradient]
  \label{c_grad_bound_ggamma}
  Let $X\in\{0,1\}^d$ follow the \emph{true} pairwise distribution
  with interaction matrix $\Thstar = S^\star+L^\star$.  Let $\kappa
  \geq1$ and let $n > c \kpa d\log d \|\Phi^\star\|^{-1}$.  Then, it
  holds with probability at least $1-d^{-\kappa}$ that
  \begin{align*}
    \ganorm{\D \nabla \ell(\Thstar)} \leq\sqrt{\frac{c \kpa d\log d
        \|\Phi^\star\|}{n}}\frac{\omega}{\xi(T)}.
  \end{align*}
\end{corollary}
\begin{proof}
  The gradient is given by
  \[
  \nabla \ell(\Thstar) = \nabla \bPr{a(\Theta) - \scalp{\Theta,
      \Phi^n}}_{|\Theta=\Thstar} = \IE_{\Thstar}[\Phi] - \Phi^n =
  \Phi^\star -\Phi^n.
  \]
  Now, let $\delta_n = \sqrt{\frac{c \kpa d\log d
      \|\Phi^\star\|}{n}}$.  Then,
  \begin{align*}
    \IP\bPr{\ganorm{\D \nabla \ell(\Thstar)} > \frac{\omega}{\xi(T)}\delta_n }
    &\leq \IP\bPr{\bNorm{\nabla \ell(\Thstar)} > \delta_n } \\
    &=\IP\bPr{\bNorm{\Phi^\star -\Phi^n} > \delta_n } \\
    &\leq d^{-\kpa},
  \end{align*}
  where the first inequality follows from $\ganorm{\D \nabla
    \ell(\Thstar)}\leq \frac{\omega}{\xi(T)}\| \nabla \ell(\Thstar)\|$
  by Lemma~\ref{lem:normcompg2}, and the last inequality follows from
  Corollary~\ref{coro_probdeltan} with $m=d$, $\Sigma^n =\Phi^n$,
	and $\Sigma = \Phi^\star$ since we have
  $\|X\|^2\leq d$ for $X\in\{0,1\}^d$ and because of the lower bound
  on $n$.  The claim follows.
\end{proof}

\subsubsection{Completion of the proof}

Now, we are ready to complete the proof of
Theorem~\ref{thm_consistency}.
\medskip

\begin{proof}[proof of Theorem~\ref{thm_consistency}]
  The lower bound
  \[
  n > c \kpa d\log d \max\bPc{\|\Phi^\star\|^{-1}, \|\Phi^\star\|
    \frac{\omega^2}{\xi(T)^2} \bPe{\frac{\alpha\nu}{32(3-\nu)}
      \min\left\{\frac{c_1}{2}, \frac{\alpha\nu\xi(T)}{128
        c_0(3-\nu)}\right\}}^{-2}}
  \]
  on $n$ as specified in Section~\ref{app:constants} implies the lower
  bound on $n$ required by Lemma~\ref{lem_upboundlan} as well as the
  lower bound required by Corollary~\ref{c_grad_bound_ggamma}. Hence,
  it follows from Corollary~\ref{c_grad_bound_ggamma} that it holds
  with probability at least $1-d^{-\kpa}$ that
  \[
    \ganorm{\D \nabla \ell(\Thstar)} \leq \sqrt{\frac{c \kpa d\log d
        \|\Phi^\star\|}{n}}\frac{\omega}{\xi(T)} =
    \frac{\nu}{6(2-\nu)}\la_n.
  \]
  Thereby, we have shown that also the last assumption from
  Proposition~\ref{prop_tspaceerrbound} and
  Proposition~\ref{prop_strictdualfeas} holds with high probability.
  Thus, provided that the precise assumptions from
  Appendix~\ref{app:constants} hold, the solution to
  Problem~\ref{prob_SL} is algebraically consistent in light of
  Proposition~\ref{prop:aconsistM} and parametrically consistent in
  the sense that $\ganorm{(\Delta_S,
    \Delta_L)}\leq\frac{32(3-\nu)}{3\alpha(2-\nu)} \la_n$ by
  Proposition~\ref{prop_tspaceerrbound}.  This finishes the proof of
  Theorem~\ref{thm_consistency}.
\end{proof}

\appendix

\section{Maximum Entropy -- Maximum Likelihood Duality}
\label{s_appendix_dual}

\paragraph{Duality of the two-sided entropy problem.}

We show that that dual of the relaxed maximum entropy problem
\[
\max_{p\,\in\, \mathcal{P}}\: H(p) \quad \textrm{ s.t. }\:
\|\mathbb{E}[\Phi] - \Phi^n\|_\infty \leq c \,\textrm{ and }\,
\|\mathbb{E}[\Phi] - \Phi^n\| \leq \lambda,
\]
where the expectation is \wrt the probability distribution $p$
on $\mathcal{X}$, is the following regularized log-likelihood
maximization problem
\[
\max_{S,\,L_1,\,L_2\,\in\, \textrm{Sym}(d)} \: \ell (S+L_1-L_2) - c\|S\|_1 -
\lambda \tr (L_1+L_2) \quad \textrm{ s.t. }\: L_1,L_2 \succeq
0. \tag{SL} \label{SL_indefinite}
\]
\medskip

We reformulate the maximum entropy problem by splitting the norm
constraints:
\begin{align*}
    \underset{p\, \geq\, 0} {\min} &\quad -H(p) & \\
    \textrm{s.t.} &\quad  \Phi^n - \IE[\Phi]  \leq c\oneVec &  (S^+)\\
		&\quad \IE[\Phi] - \Phi^n \leq c\oneVec & (S^-)\\
    &\quad \Phi^n - \IE[\Phi]  \preceq \lambda \Id& (L_1)\\
    &\quad \IE[\Phi] - \Phi^n  \preceq \lambda \Id& (L_2)\\
    &\quad \sum_{x\in\mathcal{X}} p(x) = 1& (\theta_0)
\end{align*}
Here $L_1, L_2\succeq 0$ as well as $S^+, S^{-}\geq0$ and $\theta_0$ are the
dual variables for the constraints and $\oneVec$ denotes the ($d\times
d$)-matrix whose entries are all $1$.
\medskip

We get the dual problem from the Lagrangian $\Lc = \Lc(p, S^+, S^-,
L_1, L_2, \theta_0)$ given by
\begin{align*}
  \Lc &= -H(p) + \theta_0 \bPr{\sum_{x\in\mathcal{X}} p(x)-1}
  - c \scalp{S^+ + S^-, \oneVec} + \scalp{S^+ - S^-, \Phi^n - \IE[\Phi]} \\
  &\qquad\;\; -\lambda \scalp{L_1+L_2, \Id} + \scalp{L_1-L_2, \Phi^n -
    \IE[\Phi]}\\ 
  &= -H(p) + \theta_0 \bPr{\sum_{x\in\mathcal{X}} p(x)-1}
  -c \scalp{|S|,  \oneVec} -\lambda \tr (L_1+L_2)
	 + \scalp{S + L_1 - L_2, \Phi^n - \IE[\Phi]}\\
  &= -H(p) + \theta_0 \bPr{\sum_{x\in\mathcal{X}} p(x)-1} - c \|S\|_1
   -\lambda \tr (L_1+L_2) +\scalp{\Theta, \Phi^n - \IE[\Phi]},  
\end{align*}
where we defined $S = S^+ -S^-$ in the second equality. Moreover, the
second equality follows from $|S| = S^+ + S^-$. This holds since for
an optimal solution at least one of the corresponding entries in $S^+$
and $S^-$ must be zero. For the third equality, we set $\Theta=S + L_1
- L_2$ and used that $\|S\|_1 =\scalp{|S|, \oneVec}$.
\medskip

The discrete distribution $p$ is given by the vector of probabilities
$\bPr{p(x)}_{x\in\X}$. The saddle point condition for the Lagrangian
for $p(x)$ for fixed $x\in\X$ implies that
\begin{align*}
  0 \stackrel{!}{=}\frac{\partial \Lc }{\partial p(x)}
  &= \frac{\partial }{\partial p(x)}\bPr{-H(p)
    +\theta_0 \left(\sum_{x'\in\mathcal{X}}p(x')-1\right)
    - \scalp{\Theta, \IE[\Phi]}} \\
  &=\frac{\partial }{\partial p(x)}\bPr{\sum_{x'\in\X} p(x')\log(p(x'))
    + \theta_0 \sum_{x'\in\X}p(x') -
    \scalp{\Theta, \sum_{x'\in\X} p(x')\Phi(x')}} \\
  &= \log(p(x)) + 1 + \theta_0 - \scalp{\Theta, \Phi(x)}
\end{align*}
Therefore, the parametric form of $p$ is given as
\[
  p(x) = \exp\bPr{\scalp{\Theta, \Phi(x)} - a(\Theta)},
\]
where $a(\Theta) = \theta_0+1$ is the log-partition function, that is,
the normalizer.  Using this parametric form, the negative entropy can
be expressed as
\begin{align*}
  -H(p) &= \IE [\log p] \\
  &= -a(\Theta) +  \scalp{\Theta, \IE[\Phi]} \\
  &= \ell(\Theta)  - \scalp{\Theta, \Phi^n} + \scalp{\Theta, \IE[\Phi]} \\
  &= \ell(\Theta) - \scalp{\Theta, \Phi^n - \IE[\Phi]},
\end{align*}
where the log-likelihood is given as $\ell(\Theta) = \scalp{\Theta,
  \Phi^n} - a(\Theta)$.  Plugging this expression for the negative
entropy into the Langrangian and using that it holds
$\sum_{x\in\mathcal{X}} p(x) = 1$ for the normalized distribution $p$
yields the dual function
\begin{align*}
  g(S, L_1, L_2)
  &= \ell(\Theta) - c\|S\|_1- \lambda \tr (L_1+L_2) \\
  &= \ell(S + L_1 - L_2) - c\|S\|_1 - \lambda \tr (L_1+L_2). 
\end{align*}
The dual problem is to maximize this dual function subject to the
constraints $L_1, L_2\succeq 0$. This is exactly the claimed
regularized log-likelihood maximization problem.

\paragraph{Dropping one of the spectral norm constraints.}

If we drop one of the constraints on the spectrum and keep only the
constraint $\Phi^n - \IE[\Phi] \preceq \lambda \Id$, then all terms
related to the other dropped constraint disappear. In particular, the
Lagrange multiplier $L_2$ is removed from the dual function and thus
from the corresponding dual regularized log-likelihood maximization
problem.

\section{Marginal Conditional Gaussian Model}
\label{s_appendix_marg}

Recall our definitions of $S\in \textrm{Sym}(d)$, $R\in\IR^{l\times
	d}$ and $0 \prec \Lambda\in \textrm{Sym}(l)$, and let
\[
\Gamma = \left( \begin{array}{cc} 2S & R^\top \\ R &
-\Lambda \end{array}\right).
\]
For $(x,y)\in \mathcal{X}\times\mathcal{Y} = \{0,1\}^d\times\IR^{l}$
the density of the conditional Gaussian model is then given as
\begin{align*}
	p(x,y) &\propto \exp \left( \frac{1}{2} (x, y)^\top\Gamma (x,y) 
	\right) \\
	&= \exp \left(x^\top S x + y^\top Rx - \frac{1}{2} y^\top\Lambda y
	\right) \\
	&= \exp\bPr{x^\top S x +\frac{1}{2}x^\top R^\top\Lambda^{-1} Rx
		- \frac{1}{2} (y-\Lambda^{-1}Rx)^\top\Lambda (y-\Lambda^{-1}Rx)}.
\end{align*}
For fixed values of $x$ this is the unnormalized density of a
multivariate Gaussian with mean vector $\Lambda^{-1}Rx$ and covariance matrix
$\Lambda^{-1}$.  By integrating over the Gaussian variables $y$ we get
the marginal distribution
\begin{align*}
	p(x) &= \int p(x,y) dy \\
	&\propto \exp \left( x^\top S x +\frac{1}{2}x^\top R^\top \Lambda^{-1} Rx
	\right) \int \exp \left( - \frac{1}{2} (y-\Lambda^{-1}Rx)^\top\Lambda (y-\Lambda^{-1}Rx)
	\right) dy \\
	&\propto \exp\bPr{ x^\top \left( S+\frac{1}{2} R^\top
		\Lambda^{-1} R\right)x } \\
	&=  \exp \left( \Big\langle{S+\frac{1}{2}R^\top \Lambda^{-1} R,
		x x^\top}\Big\rangle  \right)\\
	&=  \exp \left( \Big\langle{S+\frac{1}{2}R^\top \Lambda^{-1} R,
		\Phi(x)}\Big\rangle \right).
\end{align*}

\end{document}